\documentclass[accepted]{uai2026} % after acceptance, for a revised version

\usepackage[american]{babel}

\usepackage{natbib}
\usepackage{mathtools}
\usepackage{booktabs}
\usepackage{tikz}

\usepackage{amsfonts,amssymb,amsthm}
\usepackage{bbm}
\usepackage{multirow}
\usepackage{algorithm}
\usepackage{algorithmic}
\usepackage{paralist}
\setdefaultleftmargin{1em}{}{}{}{}{}
\usepackage{comment}

\newtheorem{theorem}{Theorem}
\newtheorem{definition}{Definition}

\newtheorem{assumption}{Assumption}

\newcommand{\R}{\mathbb{R}}
\newcommand{\E}{\mathbb{E}}

\newcommand{\Cov}{\mathrm{Cov}}

\newcommand{\bx}{\mathbf{x}}
\newcommand{\by}{\mathbf{y}}
\newcommand{\bz}{\mathbf{z}}

\newcommand{\bK}{\mathbf{K}}
\newcommand{\bP}{\mathbf{P}}
\newcommand{\bW}{\mathbf{W}}

\newcommand{\bm}{\mathbf{m}}

\newcommand{\bphi}{\boldsymbol{\phi}}

\newcommand{\bPhi}{\boldsymbol{\Phi}}
\newcommand{\beps}{\boldsymbol{\epsilon}}

\newcommand{\bmu}{\boldsymbol{\mu}}
\newcommand{\bSigma}{\boldsymbol{\Sigma}}

\newcommand{\bdelta}{\boldsymbol{\delta}}

\newcommand{\cL}{\mathcal{L}}
\newcommand{\cD}{\mathcal{D}}
\newcommand{\cH}{\mathcal{H}}
\newcommand{\cN}{\mathcal{N}}
\newcommand{\cO}{\mathcal{O}}
\newcommand{\cF}{\mathcal{F}}
\newcommand{\cG}{\mathcal{G}}
\newcommand{\cS}{\mathcal{S}}

\title{Nonlocal Bayesian Modeling of Continuous Spatio-Temporal Dynamics}

\author[1]{\href{mailto:dlwodud116@kaist.ac.kr}{Jaeyeong~Lee}{}}
\author[1]{Heeyoung~Kim\thanks{Corresponding author: \texttt{heeyoungkim@kaist.ac.kr}}}
\affil[1]{%
    Department of Industrial and Systems Engineering\\
    Korea Advanced Institute of Science and Technology (KAIST)\\
    Daejeon, Republic of Korea
}

\begin{document}
\maketitle

\begin{abstract}
Real-world spatio-temporal forecasting must handle irregular time points, spatially sparse observations, and the need for uncertainty quantification. This setting is often further compounded by \emph{nonlocal} interactions (long-range spatial coupling). Modeling continuous-space, continuous-time nonlocal dynamics naturally leads to infinite-dimensional integro-differential equations (IDEs), making principled Bayesian inference intractable. We propose the NonLocal Bayesian Spatio-Temporal model (\textsc{NLBST}), a hierarchical Bayesian framework for continuous spatio-temporal fields that learns explicit nonlocal coupling while retaining tractable inference. \textsc{NLBST} represents the latent field via a coordinate-based spatial basis expansion and models the coefficient process with a continuous-time ODE whose learnable linear operator corresponds to a Galerkin reduction of a nonlocal IDE; a Neural ODE residual captures additional nonlinear dynamics. 
A linear-Gaussian observation model enables Kalman-style sequential updates under missing and irregular observations, while the spatial basis representation enables inductive prediction at unmeasured locations without retraining. Global parameters are learned via variational inference, and uncertainty is handled through a Bayesian hierarchy. Experiments on synthetic and real-world datasets demonstrate strong forecasting and spatial generalization with well-calibrated uncertainty, yielding substantial gains over baselines in strongly nonlocal and partially observed regimes.

\end{abstract}

\section{Introduction}
\label{sec:intro}

Spatio-temporal forecasting is central to applications ranging from climate science to public health \citep{faghmous2014spatio,kim2019spatiotemporal,kumar2024spatio}. However, real-world sensing imposes practical constraints: observations arrive at irregular time points, are spatially sparse and scattered, and require calibrated uncertainty for downstream decision-making \citep{rubanova2019latent, jiang2018survey, pawar2024handling}. Addressing these three challenges in a single continuous space-time model with principled uncertainty quantification remains difficult. 

Existing approaches address these challenges only partially. 
Classical statistical models \citep{cressie2015statistics, wikle2019spatio, zammit2022deep}, such as kriging and Gaussian processes, naturally accommodate irregularly located observations and provide principled uncertainty; however, they often rely on separable or weakly structured spatio-temporal covariances, limiting their ability to represent complex, dynamically evolving dependencies.
Dynamic spatio-temporal models (DSTMs) \citep{stroud2001dynamic, wikle2010general, koo2024deep} explicitly model temporal evolution through state-transition equations, but their discrete-time Markov structure limits applicability under irregular temporal sampling. 
Deep learning methods \citep{yu2017spatio, che2018recurrent, jin2023spatio, yalavarthi2024grafiti, zhang2024irregular, liu2026apn} achieve strong forecasting accuracy at fixed sensor locations, but are often spatially transductive and rely on heuristic uncertainty estimates. 
Continuous-time neural models \citep{rubanova2019latent, kidger2020neural, choi2022graph, wu2024continuously}, typically based on Neural ODEs \citep{chen2018neural}, handle irregular sampling yet lack explicit spatial structure or principled Bayesian inference.

\begin{figure}[t]
    \centering
    \includegraphics[width=0.85\columnwidth]{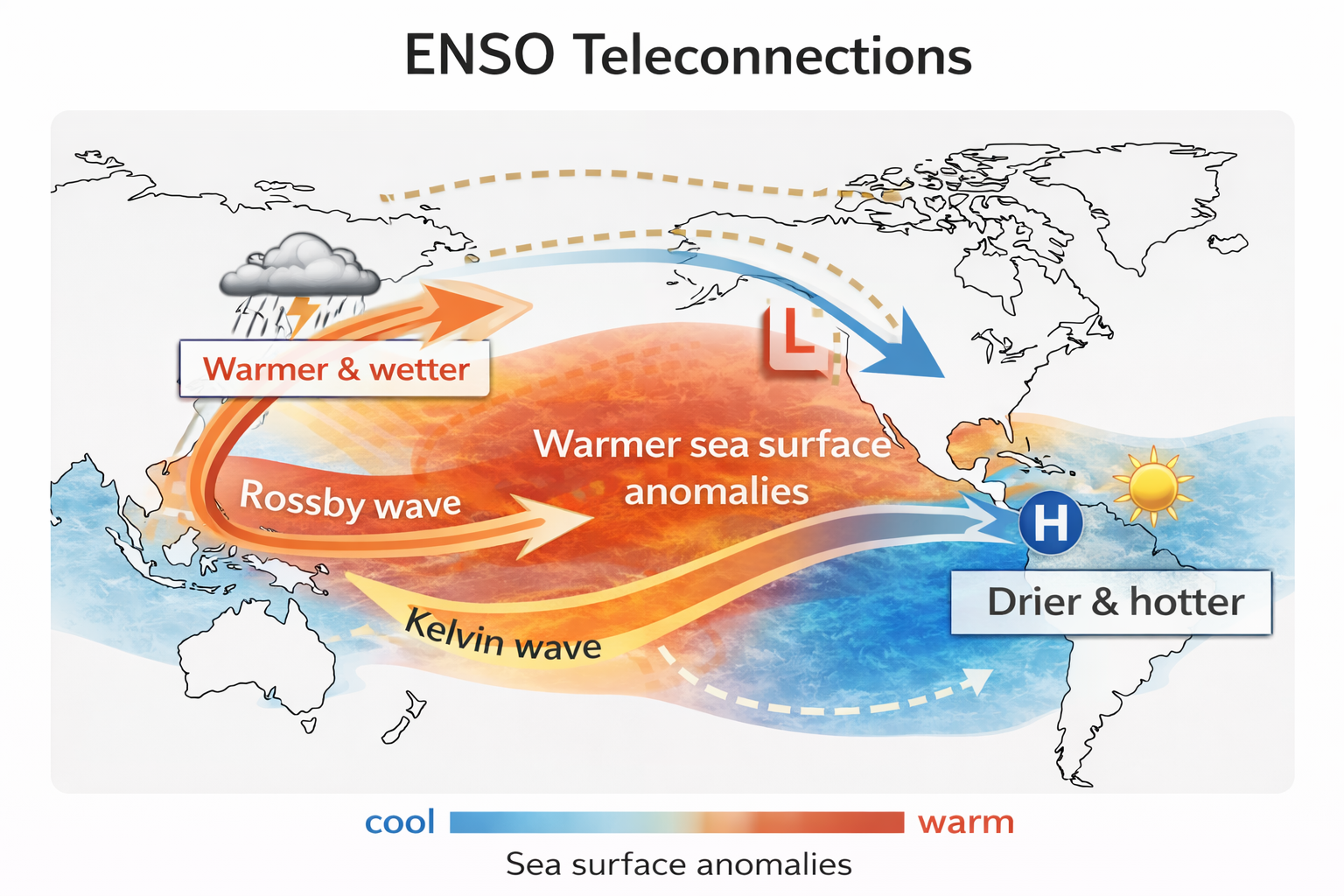}
    \caption{Schematic of ENSO teleconnections illustrating nonlocal spatio-temporal interactions. Tropical Pacific SST anomalies propagate through large-scale wave adjustments (Rossby and Kelvin waves), inducing coherent responses in distant regions and exemplifying long-range dependence beyond local interactions.}
    \label{fig:teleconnection}
\end{figure}

Furthermore, many spatio-temporal systems exhibit \textit{nonlocal} interactions, in which the temporal evolution at one location depends not only on its immediate neighborhood but also on distant regions \citep{wallace1981teleconnections, lee2020bayesian, tibau2022spatiotemporal,jung2024spatially}. A canonical example is climate teleconnection \citep{alizadeh2024review}: sea surface temperature (SST) anomalies in the tropical Pacific associated with the El Ni\~no--Southern Oscillation (ENSO) can reorganize atmospheric circulation and induce coherent responses in remote regions \citep{yang2018nino}, including precipitation anomalies over East Africa, storm-track shifts over the North Pacific, and hydroclimate variability over western North America \citep{yeh2018enso} (Figure~\ref{fig:teleconnection}). 

A principled formulation of such nonlocal dynamics in continuous space-time is given by the integro-differential equation (IDE) \citep{laing2003pde}:
\begin{equation}
\label{eq:ide_intro}
\partial_t \mu(x,t) = \int_D G(x,x') \mu(x',t) \, dx' + \mathcal{F}[\mu](x,t),
\end{equation}
where the kernel $G(x,x')$ encodes how location $x'$ influences the evolution at $x$, and $\mathcal{F}$ represents additional local dynamics.
However, posterior inference is typically intractable: in continuous space, $\mu(\cdot,t)$ is an infinite-dimensional function, and the nonlocal integral term is generally not available in closed form, requiring fine discretization of the domain and expensive numerical evaluation \citep{stuart2010inverse, dashti2015bayesian}.

To address these challenges, we propose a NonLocal Bayesian Spatio-Temporal model (NLBST),
a hierarchical Bayesian framework for spatio-temporal fields over continuous space and time that learns nonlocal interactions through a coupled spatio-temporal structure. 
The proposed model decomposes the latent field into coordinate-based spatial basis functions and a time-varying coefficient vector whose dynamics are governed by an ODE in continuous time. Irregular temporal sampling is handled by propagating the coefficient state across arbitrary time gaps, while spatial prediction at unmeasured locations reduces to evaluating the basis functions at new coordinates. 

The coefficient dynamics include an explicit linear coupling term that encodes nonlocal interactions. Conceptually, this can be viewed as a Galerkin reduction of the IDE in Eq.\eqref{eq:ide_intro}, providing a tractable finite-rank representation of long-range dependence, with all inference performed in coefficient space without discretizing the field. A Neural ODE residual \citep{chen2018neural} augments this linear term to capture additional nonlinear dynamics. For uncertainty quantification, a linear-Gaussian observation model admits closed-form Kalman-style updates for sequential Bayesian filtering, while global parameters are learned via variational inference over the Bayesian hierarchy \citep{wikle1998hierarchical}.

Our contributions are as follows:
\begin{compactitem}
    \item \textbf{Continuous-time Bayesian filtering for irregular and incomplete observations.}
    We propagate the latent coefficient state via an ODE across arbitrary time gaps. Combined with a linear-Gaussian measurement model, this yields closed-form Kalman-style measurement updates that naturally handle missing entries without the need for imputation.

    \item \textbf{Inductive spatial prediction via spatial basis representation.}
    Since the learned coefficient dynamics are shared across space, prediction at unmeasured locations requires only evaluating the basis functions at new coordinates, without retraining.

    \item \textbf{Principled uncertainty via Bayesian hierarchy.}
    Uncertainty is captured through a Bayesian hierarchical model: a linear-Gaussian observation model allows closed-form Kalman-style updates for sequential filtering, while global parameters are learned via variational inference.

    \item \textbf{Learnable nonlocal interactions via a Galerkin reduction.}
    The linear coupling term in the coefficient ODE induces a finite-rank nonlocal kernel in field space---a Galerkin reduction of a nonlocal IDE---providing an explicit and interpretable pathway for long-range spatial dependence.
\end{compactitem}

We validate NLBST on synthetic dynamics and real-world environmental datasets, demonstrating competitive forecasting accuracy, reliable spatial generalization to unmeasured locations, and well-calibrated predictive uncertainty---particularly excelling in settings with nonlocal interactions and irregular temporal sampling. The code for NLBST is available at \url{https://github.com/JaeyeongLee1/NLBST}.

\section{Methodology}
\label{sec:model}

\subsection{Problem Definition}

Let $D \subset \R^d$ be a spatial domain. We observe measurements at a set of sensor locations $\cS_m = \{x^m_i\}_{i=1}^{S_m} \subset D$ and consider a set of unmeasured locations $\cS_u = \{x^u_j\}_{j=1}^{S_u} \subset D$. 
Let $\{t_k\}_{k=1}^{T}$ denote the observation times (not necessarily equally spaced), and let $\by_k \in \R^{S_m}$ denote the vector of measurements at time $t_k$ across the measured locations.
We allow missing entries in $\by_k$ due to sensor failure, represented by a binary mask $\bdelta_k \in \{0,1\}^{S_m}$, where $\bdelta_{k,i}=1$ indicates that $y_{k,i}$ is missing. 
Our goals are:
\begin{compactitem}
    \item[(i)] \textbf{Forecasting}: predict future values at measured locations for $t > t_T$;
    \item[(ii)] \textbf{Spatial prediction}: infer the signal at $x \in \cS_u$ without retraining;
    \item[(iii)] \textbf{Uncertainty quantification}: provide calibrated predictive uncertainty for both tasks.
\end{compactitem}
To achieve these goals, we adopt a three-level Bayesian hierarchy comprising a data model, a process model, and a parameter model.

\subsection{Level 1: Data Model}

Let $y(x,t)$ be an observation at coordinate $x \in D$ and time $t$.
We assume the observations are noisy measurements of an underlying spatio-temporal latent process:
\begin{equation}
\label{eq:data_model_general}
y(x,t) = \mu(x,t) + \epsilon(x,t),
\qquad
\epsilon(x,t) \stackrel{\mathrm{iid}}{\sim} \cN(0,\sigma_{\mathrm{obs}}^2),
\end{equation}
where $\mu(x,t)$ is a latent spatio-temporal field and $\sigma_{\mathrm{obs}}^2$ is the observation noise variance.

\textbf{Vector form at measured locations.}
At time $t_k$, let $\cS_k \subseteq \cS_m$ denote the set of locations with available measurements.
Stacking the observations yields
%\vspace{-0.5em}
\begin{equation}
\label{eq:data_model_stack}
\by_k = \bmu_k + \beps_k,
\qquad
\beps_k \sim \cN(\mathbf{0},\sigma_{\mathrm{obs}}^2 I_{n_k}),
\end{equation}
where $n_k = |\cS_k|$ and $\bmu_k = \big(\mu(x,t_k)\big)_{x \in \cS_k}$.

\subsection{Level 2: Process Model}

We specify the latent process $\mu(x,t)$ through a spatial basis representation with continuous-time coefficient dynamics. This coupled spatio-temporal structure addresses the challenges of inductive spatial prediction and irregular temporal sampling, while naturally enabling a learnable nonlocal interaction mechanism.

\subsubsection{Spatial Basis Representation}

To enable prediction at arbitrary, unmeasured locations, we represent the latent field as a basis-function expansion over continuous coordinates:
\begin{equation}
\label{eq:basis_decomp}
\mu(x,t) = \bphi(x)^\top \bz(t),
\end{equation}
where $\bphi(x) = (\phi_1(x), \ldots, \phi_K(x))^\top \in \R^K$ is a vector of $K$ spatial basis functions evaluated at location $x \in D$, and $\bz(t) \in \R^K$ is a time-varying coefficient vector.

\subsubsection{Continuous-Time Coefficient Dynamics}

To handle irregular temporal sampling, we model the coefficient process in continuous time via an ODE:
%\vspace{-0.5em}
\begin{equation}
\label{eq:coeff_ode_full}
\frac{d\bz(t)}{dt} = A_{\mathrm{NL}} \bz(t) + g_\theta(\bz(t), t),
\end{equation}
where $A_{\mathrm{NL}} \in \R^{K \times K}$ is a learnable linear matrix and $g_\theta: \R^K \times \R \to \R^K$ is a neural network parameterized by $\theta$. 
This continuous-time formulation naturally accommodates irregular observation times: given timestamps $\{t_k\}_{k=1}^T$, we obtain the coefficient state at time $t_k$ by integrating the ODE from the previous state:
%\vspace{-0.5em}
\begin{equation}
\label{eq:ode_flow}
\hat{\bz}_k = \mathrm{ODESolve}\!\left(\bz \mapsto A_{\mathrm{NL}}\bz + g_\theta(\bz,t),\; \bz_{k-1},\; t_{k-1},\; t_k\right).
\end{equation}
The integration adapts to the arbitrary time increment $\Delta t_k = t_k - t_{k-1}$ rather than requiring a fixed step size, eliminating the need for imputation.
Details of the state update mechanism is provided in Sec.~\ref{sec:inference}.

\subsubsection{Induced Nonlocal Coupling Structure}

The proposed coupled spatio-temporal modeling induces an explicit nonlocal interaction structure in field space. 
To make this precise, we express the temporal evolution of the field $\mu(x,t)$ in terms of the coefficient dynamics. From Eq.\eqref{eq:basis_decomp} and Eq.\eqref{eq:coeff_ode_full}, the linear component of the field evolution is:
\begin{equation}
\label{eq:field_evolution_linear}
\partial_t \mu(x,t) \big|_{\text{linear}} = \bphi(x)^\top A_{\mathrm{NL}} \bz(t).
\end{equation}
Under the spatial basis representation in Eq.\eqref{eq:basis_decomp}, the coefficients are given by
\begin{equation}
\label{eq:galerkin_coeff}
\bz(t) = M^{-1} \int_D \bphi(x') \mu(x',t) \, dx',
\end{equation}
where $M \in \R^{K \times K}$ is the mass matrix with entries $M_{ij} = \int_D \phi_i(x) \phi_j(x) \, dx$. 
Substituting Eq.\eqref{eq:galerkin_coeff} into Eq.\eqref{eq:field_evolution_linear} yields:
\begin{align}
\partial_t \mu(x,t) \big|_{\text{linear}} 
&= \bphi(x)^\top A_{\mathrm{NL}} M^{-1} \int_D \bphi(x') \mu(x',t) \, dx' \nonumber \\
&= \int_D \underbrace{\bphi(x)^\top A_{\mathrm{NL}} M^{-1} \bphi(x')}_{G(x,x')} \mu(x',t) \, dx'. \label{eq:effective_kernel}
\end{align}
This has precisely the form of a nonlocal IDE:
\begin{equation}
\label{eq:ide_field}
\partial_t \mu(x,t) = \int_D G(x,x') \mu(x',t) \, dx' + \cF[\mu](x,t),
\end{equation}
where the induced kernel $G(x,x') {=} \bphi(x)^\top A_{\mathrm{NL}} M^{-1} \bphi(x')$ encodes nonlocal spatial interactions, and the neural residual $g_\theta$ corresponds to the additional dynamics $\cF$. 
Therefore, learning the coefficient-space matrix $A_{\mathrm{NL}}$ is equivalent to learning a finite-rank nonlocal kernel $G$ in field space. 
From this perspective, the proposed model can be viewed as a Galerkin reduction of a nonlocal IDE.

\subsubsection{Stochastic State-Space Formulation}

To account for model mismatch and latent process uncertainty, we model the coefficient dynamics as an additive-noise diffusion:
\begin{align}
\label{eq:sde_coeff_em}
d\bz(t) &= f(\bz(t),t)\,dt + \sigma_{\mathrm{proc}}\, d\bW_t, \nonumber\\
f(\bz,t) &= A_{\mathrm{NL}}\bz + g_\theta(\bz,t),
\end{align}
where $\bW_t$ is a $K$-dimensional Wiener process. 
Given irregular observation times $\{t_k\}$, we use a first-order diffusion discretization to define a tractable Gaussian transition approximation:
\begin{align}
\label{eq:em_transition}
\hat{\bz}_k
&= \mathrm{ODESolve}\!\big(\bz \mapsto f(\bz,t),\, \bz_{k-1},\, t_{k-1},\, t_k\big), \nonumber\\
\bz_k \mid \bz_{k-1}
&\approx \cN\!\left(\hat{\bz}_k,\; \sigma_{\mathrm{proc}}^2 \Delta t_k \, I_K \right).
\end{align}
The isotropic diffusion covariance $\sigma_{\mathrm{proc}}^2 \Delta t_k I_K$ is a simplifying choice that yields a stable and efficient continuous-discrete filtering scheme under irregular sampling.
Since $\sigma_{\mathrm{proc}}$ is learned within the variational objective, the effective process uncertainty is calibrated end-to-end to best explain the data likelihood under the chosen variational family.

\subsubsection{Spatially Inductive Prediction}

The basis representation enables prediction at arbitrary, unmeasured locations $\cS_u$.
Let $\bPhi(\cS_u) \in \R^{S_u \times K}$ be the basis matrix evaluated at $\cS_u$.
The predictive distribution at $\cS_u$ is obtained via linear uncertainty propagation:
\begin{align}
\E[\bmu_u(t) \mid \cD] &= \bPhi(\cS_u) \, \E[\bz(t) \mid \cD], \label{eq:unmeasured_mean} \\
\Cov(\bmu_u(t) \mid \cD) &= \bPhi(\cS_u) \, \Cov(\bz(t) \mid \cD) \, \bPhi(\cS_u)^\top. \label{eq:unmeasured_cov}
\end{align}
Including observation noise, the predictive covariance becomes $\Cov(\by_u(t) \mid \cD) = \Cov(\bmu_u(t) \mid \cD) + \sigma_{\mathrm{obs}}^2 I_{S_u}$.

\subsubsection{Measurement Model in State-Space Form}
\label{sec:filter_overview}

At observation time $t_k$, combining Eq.\eqref{eq:data_model_stack} 
with Eq.\eqref{eq:basis_decomp} yields the linear-Gaussian measurement model:
\begin{equation}
\label{eq:meas_coeff}
\by_k = \bPhi_k \bz_k + \beps_k,
\quad
\beps_k \sim \cN(\mathbf{0}, \sigma_{\mathrm{obs}}^2 I_{n_k}),
\end{equation}
where $\bPhi_k := \bPhi(\cS_k)$ contains rows $\bphi(x)^\top$ for observed 
locations $x \in \cS_k$.
Since $\by_k$ depends linearly on $\bz_k$ with additive Gaussian noise, the filtering update for $p(\bz_k \mid \by_{1:k})$ admits a closed-form Kalman-style conditioning at each observation time, naturally accommodating irregular sampling and missingness. Details of the full predict--update recursion and parameter learning are provided in Sec.~\ref{sec:inference}.

\subsection{Level 3: Parameter Model}

Let $\Theta$ denote all model parameters: the nonlocal coupling matrix $A_{\mathrm{NL}} \in \R^{K \times K}$, Neural ODE weights $\theta$, noise scales $\sigma_{\mathrm{obs}}$ and $\sigma_{\mathrm{proc}}$, initial state $\bz_0$, and spatial basis parameters (centers $\{c_r\}$ and lengthscale $\ell$ when learnable). 
The structural parameters $A_{\mathrm{NL}}$, $\theta$, and the spatial basis parameters are point-estimated via an evidence lower bound (ELBO) objective, while we place weakly informative priors on the remaining parameters:
\begin{align}
\sigma_{\mathrm{obs}} &\sim \mathrm{LogNormal}(\mu_{\mathrm{obs}}, 
  \tau_{\mathrm{obs}}^2), \\
\sigma_{\mathrm{proc}} &\sim \mathrm{LogNormal}(\mu_{\mathrm{proc}}, 
  \tau_{\mathrm{proc}}^2), \\
\bz_0 &\sim \cN(\mathbf{0}, \sigma_0^2 I_K).
\end{align}

\section{Inference}
\label{sec:inference}

Given observations $\cD=\{(t_k,\by_k)\}_{k=1}^{T}$ at measured locations, with missingness indicators $\{\bdelta_k\}$, our goal is to approximate the posterior over the latent coefficient states $\{\bz_k\}_{k=1}^{T}$ and learn the model parameters. 
Our model combines a continuous-time transition---comprising a linear nonlocal component and a nonlinear Neural ODE residual---with a linear-Gaussian measurement model.
This structure enables analytic updates in the observation space while retaining flexible nonlinear dynamics.

In practice, we adopt a Kalman-style variational approximation for the latent states:
\begin{multline}
q(\bz_{1:T}, \sigma_{\mathrm{obs}}, \sigma_{\mathrm{proc}}) 
= q(\sigma_{\mathrm{obs}})\,q(\sigma_{\mathrm{proc}})\\
\times\prod_{k=1}^{T} q(\bz_k \mid \cD_{1:k}, \sigma_{\mathrm{obs}}, \sigma_{\mathrm{proc}}),
\end{multline}
where each $q(\bz_k \mid \cdot)$ is a Gaussian distribution whose moments are obtained via a closed-form predict--update recursion implied by the linear-Gaussian measurement model in Eq.\eqref{eq:meas_coeff}.
The global noise scales $(\sigma_{\mathrm{obs}},\sigma_{\mathrm{proc}})$ are treated as latent variables with variational distributions.
The nonlocal kernel matrix $A_{\mathrm{NL}}$, Neural ODE parameters $\theta$, and basis parameters are learned via gradient-based optimization using stochastic variational inference (SVI).

\subsection{Variational Filtering for Latent States}
\label{sec:structured_filtering}

Let $\bz_k := \bz(t_k)$ denote the latent coefficient state at time $t_k$.
Conditioned on $(\sigma_{\mathrm{obs}},\sigma_{\mathrm{proc}})$ and the model parameters, 
\[
q(\bz_k \mid \cD_{1:k}) \approx \cN(\bm_k^+,\bP_k^+),
\]
where $(\bm_k^+,\bP_k^+)$ are obtained by alternating a continuous-time prediction step with an analytic observation update. Latent samples $\bz_k$ drawn from the filtered posterior are used to evaluate the ELBO (Eq.\eqref{eq:elbo_inference}). 
Algorithm~\ref{alg:filtering} summarizes the variational filtering procedure with analytic measurement updates.

\textbf{Predict.}
Let $(\bm_{k-1}^+,\bP_{k-1}^+)$ denote the filtered moments at time $t_{k-1}$.
The predicted mean is obtained by integrating the coefficient dynamics (Eq.\eqref{eq:coeff_ode_full}):
\begin{multline}
\label{eq:q_predict_mean}
\bm_k^- = \mathrm{ODESolve}\!\big(\bz \mapsto A_{\mathrm{NL}}\bz + g_\theta(\bz, t),\\
 \bm_{k-1}^+,\; t_{k-1},\; t_k\big).
\end{multline}
To propagate uncertainty under irregular sampling, we adopt a first-order diffusion approximation within the Gaussian variational family:
%\vspace{-0.5em}
\begin{equation}
\label{eq:q_predict_cov}
\bP_k^- \approx \bP_{k-1}^+ + \sigma_{\mathrm{proc}}^2 \Delta t_k\, I_K,
\end{equation}
where $\Delta t_k = t_k - t_{k-1}$.
This approximation treats the deformation of covariance induced by the nonlinear drift as additional process noise; in practice, $\sigma_{\mathrm{proc}}$ is learned end-to-end via the variational objective to best explain the observed data under the chosen variational family.

\textbf{Update.}
At time $t_k$, let $\cS_k$ denote the subset of locations with available measurements (determined by the missingness indicator $\bdelta_k$).
Under Eq.\eqref{eq:meas_coeff}, the variational update follows the standard Kalman correction:
\begin{align}
\label{eq:q_kalman_gain}
\bK_k &= \bP_k^- \bPhi_k^\top
\!\left(\bPhi_k \bP_k^- \bPhi_k^\top + \sigma_{\mathrm{obs}}^2 I_{n_k}\right)^{-1}\!\!,\\
\label{eq:q_kalman_mean}
\bm_k^+ &= \bm_k^- + \bK_k\left(\by_k - \bPhi_k\bm_k^-\right),\\
\label{eq:q_kalman_cov}
\bP_k^+ &= \left(I - \bK_k \bPhi_k\right)\bP_k^-.
\end{align}
Missing observations are handled naturally by restricting the update to the observed subset $\cS_k$.

\textbf{Sample.}
After obtaining $(\bm_k^+,\bP_k^+)$, we draw
\begin{equation}
\label{eq:q_sample_latent}
\bz_k \sim q(\bz_k \mid \cD_{1:k}) = \cN(\bm_k^+,\bP_k^+),
\end{equation}
to estimate Monte Carlo expectations in the ELBO and to construct posterior predictive distributions (Sec.~\ref{sec:posterior_predictive}).

\subsection{Initialization from the First Observation}
\label{sec:init_posterior}

We construct an initial approximate posterior from the first observation via Bayesian linear regression, combining the prior $\bz_0 \sim \cN(\mathbf{0}, \sigma_0^2 I_K)$ with the first measurement.
Given $\by_1$ and $\bPhi_1 := \bPhi(\cS_1)$, the posterior is
\begin{align}
\bP_1^+ &= \left( \sigma_0^{-2} I_K + \sigma_{\mathrm{obs}}^{-2} \bPhi_1^\top \bPhi_1 \right)^{-1}\!\!, \label{eq:init_cov} \\
\bm_1^+ &= \sigma_{\mathrm{obs}}^{-2} \bP_1^+ \bPhi_1^\top \by_1, \label{eq:init_mean}
\end{align}
which corresponds to the standard Gaussian posterior under the linear observation model (Eq.\eqref{eq:meas_coeff}) with prior covariance $\sigma_0^2 I_K$.
This initialization is consistent with the Kalman update Eq.\eqref{eq:q_kalman_gain}--Eq.\eqref{eq:q_kalman_cov} applied to a prior predictive $\cN(\mathbf{0}, \sigma_0^2 I_K)$.

\subsection{Variational Learning of Global Parameters}
\label{sec:vi_global}

We learn global parameters by maximizing an ELBO via SVI. We use the variational family
\[
  q(\bz_{1:T},\sigma_{\mathrm{obs}},\sigma_{\mathrm{proc}})
  =
  q(\sigma_{\mathrm{obs}})\,q(\sigma_{\mathrm{proc}})\,
  q(\bz_{1:T}\mid\sigma_{\mathrm{obs}},\sigma_{\mathrm{proc}}),
\]
where $q(\bz_{1:T}\mid\cdot)$ is the structured Gaussian posterior induced by the filtering recursion (Sec. ~\ref{sec:structured_filtering}).
The ELBO is
\begin{equation}
\label{eq:elbo_inference}
  \cL_{\mathrm{ELBO}}
  =
  \E_{q}\!\Big[
    \log p(\cD,\bz_{1:T},\sigma)
    -
    \log q(\bz_{1:T},\sigma)
  \Big],
\end{equation}
where $\sigma{=}(\sigma_{\mathrm{obs}},\sigma_{\mathrm{proc}})$. The joint model factorizes according to the state-space structure in~Eq.\eqref{eq:em_transition} and Eq.\eqref{eq:meas_coeff}, with likelihoods evaluated only at observed entries.

\begin{algorithm}[t]
\caption{Variational filtering with analytic measurement update}
\label{alg:filtering}
\begin{algorithmic}[1]
\STATE \textbf{Input:} Observations $\{(t_k,\by_k)\}_{k=1}^{T}$; basis matrices $\{\bPhi_k\}$; coupling $A_{\mathrm{NL}}$; residual $g_\theta$; noise $(\sigma_{\mathrm{obs}},\sigma_{\mathrm{proc}})$; prior $\sigma_0^2$.
\STATE \textbf{Output:} Filtered $\{(\bm_k^+,\bP_k^+)\}$ and samples $\{\bz_k\}$.

\STATE \textbf{(Initialization)}
\STATE $\bP_1^+ \!\leftarrow\! \left( \sigma_0^{-2} I_K + \sigma_{\mathrm{obs}}^{-2} \bPhi_1^\top \bPhi_1 \right)^{-1}$
\STATE $\bm_1^+ \leftarrow \sigma_{\mathrm{obs}}^{-2} \bP_1^+ \bPhi_1^\top \by_1$
\STATE Sample $\bz_1 \sim \cN(\bm_1^+,\bP_1^+)$

\FOR{$k = 2$ \TO $T$}
    \STATE \textbf{(Predict)}
    \STATE $\Delta t_k \leftarrow t_k - t_{k-1}$
    \STATE $\bm_k^- \!\leftarrow\! \mathrm{ODESolve}\!\big(A_{\mathrm{NL}} \cdot {+}\, g_\theta,\, \bm_{k-1}^+,\, t_{k-1},\, t_k\big)$
    \STATE $\bP_k^- \leftarrow \bP_{k-1}^+ + \sigma_{\mathrm{proc}}^2 \Delta t_k \, I_K$

    \STATE \textbf{(Update)}
    \STATE $\bK_k \!\leftarrow\! \bP_k^- \bPhi_k^\top
    \!\left(\bPhi_k \bP_k^- \bPhi_k^\top {+}\, \sigma_{\mathrm{obs}}^2 I_{n_k}\right)^{-1}$
    \STATE $\bm_k^+ \leftarrow \bm_k^- + \bK_k\left(\by_k - \bPhi_k \bm_k^-\right)$
    \STATE $\bP_k^+ \leftarrow \left(I_K - \bK_k \bPhi_k\right)\bP_k^-$

    \STATE \textbf{(Sample)}
    \STATE Sample $\bz_k \sim \cN(\bm_k^+,\bP_k^+)$
\ENDFOR

\STATE \textbf{Return} $\{(\bm_k^+,\bP_k^+)\}_{k=1}^{T}$, $\{\bz_k\}_{k=1}^{T}$
\end{algorithmic}
\end{algorithm}

\subsection{Posterior Predictive Distribution}
\label{sec:posterior_predictive}

\textbf{Temporal forecasting.}
For forecasting at measured locations, we first run the filtering procedure over a historical context window to obtain the filtered posterior moments $(\bm_T^+,\bP_T^+)$ at the final observed time.
We then propagate the latent state forward by integrating the dynamics (Eq.\eqref{eq:coeff_ode_full}) and injecting process noise according to Eq.\eqref{eq:em_transition}, without measurement updates.
Predictive uncertainty is obtained by Monte Carlo sampling: we draw  $(\sigma_{\mathrm{obs}},\sigma_{\mathrm{proc}})$ from the variational posterior, sample $\bz(T)$ from $\cN(\bm_T^+,\bP_T^+)$, roll forward the latent dynamics via $A_{\mathrm{NL}}\bz + g_\theta(\bz, t)$, and decode the resulting latent trajectories into the observation space.

\textbf{Spatial prediction.}
For unmeasured locations $\cS_u$, we evaluate the basis matrix $\bPhi(\cS_u)$ and decode the same latent coefficients as in Eq.\eqref{eq:unmeasured_mean}. Predictive moments at $\cS_u$ are then obtained via linear uncertainty propagation (Eq.\eqref{eq:unmeasured_mean}--Eq.\eqref{eq:unmeasured_cov}).

\section{Theoretical Guarantees}
\label{sec:theory}

\textbf{Continuum view of \textsc{NLBST}.}
We interpret \textsc{NLBST} as a Galerkin approximation of the $\cH{=}L^2(D)$-valued stochastic nonlocal evolution: 
\begin{equation}\label{eq:spde}
d\mu(t) = \bigl(\cG\mu(t)+\cF_\theta(t,\mu(t))\bigr)\,dt
        + \Sigma(t,\mu(t))\,dW_t,
\end{equation}
where $(\cG u)(x){=}\int_D G(x,x')u(x')dx'$ is a Hilbert--Schmidt integral operator, $\cF_\theta$ is a pointwise nonlinear drift, and $W_t$ is a $Q$-Wiener process on $\cH$.
Projecting Eq.\eqref{eq:spde} onto $V_K{=}\mathrm{span}\{\phi_1,\dots,\phi_K\}$ yields the coefficient dynamics in  Eq.\eqref{eq:sde_coeff_em}.

\textbf{Assumptions.}
We assume globally Lipschitz drift and diffusion with linear growth, $\mathrm{Tr}(Q){<}\infty$, and a square-integrable initial condition (Assumptions~\ref{ass:bounded}--\ref{ass:initial} in  Appendix~\ref{app:theory}).

\begin{theorem}[Well-posedness]\label{thm:wellposed_main}
Under Assumptions~\ref{ass:bounded}--\ref{ass:initial}, Eq.\eqref{eq:spde} admits a unique mild solution
$\mu\in L^2(\Omega;C([0,T];\cH))$, where $\Omega$ is the underlying probability space.
\end{theorem}

\begin{theorem}[Stability]\label{thm:stability_main}
Under Assumptions~\ref{ass:bounded}--\ref{ass:initial}, for any two mild solutions
$\mu,\tilde\mu$ to Eq.\eqref{eq:spde} with initial conditions $\mu_0,\tilde\mu_0$, there exists
$C_T>0$ such that
\[
\E\Big[\sup_{t\in[0,T]}\|\mu(t)-\tilde\mu(t)\|^2\Big]
\le C_T\,\E\big[\|\mu_0-\tilde\mu_0\|^2\big].
\]
\end{theorem}

\begin{theorem}[Galerkin convergence]\label{thm:galerkin_main}
Assume additionally that $\overline{\cup_{K\ge1}V_K}{=}\cH$ (Assumption~\ref{ass:approx}).
Let $\mu_K$ denote the Galerkin solution (Eq.\eqref{eq:galerkin_projected} in Appendix~\ref{app:theory}).
Then
\[
\lim_{K\to\infty}\E\Big[\sup_{t\in[0,T]}\|\mu(t)-\mu_K(t)\|^2\Big]=0.
\]
\end{theorem}

Theorem~\ref{thm:galerkin_main} implies that increasing $K$ systematically improves \textsc{NLBST}'s finite-dimensional approximation, which we verify empirically in Sec.~\ref{sec:experiments} (Table~\ref{tab:rq4_ksweep}). Detailed analysis and proofs are provided in Appendix~\ref{app:theory}.
 
\section{Related Work}
\label{sec:related}

\textbf{Statistical spatio-temporal modeling.}
Statistical spatio-temporal models provide principled uncertainty quantification and naturally accommodate sparse spatial measurements \citep{cressie2015statistics,wikle2019spatio}. Dynamic spatio-temporal models (DSTMs) cast spatio-temporal fields in Markovian state-space form and admit exact Kalman filtering under a linear-Gaussian measurement model \citep{wikle2010general}, while SPDE-based methods link Gaussian fields to stochastic PDEs for scalable inference \citep{lindgren2011explicit,sigrist2015stochastic,clarotto2024spde}. These frameworks often rely on discrete-time transitions or pre-specified (typically local) operators, which can limit expressivity under irregular timestamps and long-range coupling. 

\textbf{Deep spatio-temporal forecasting and neural operators.}
Deep spatio-temporal forecasting models \citep{yu2017spatio, guo2019attention, liu2022msdr, zhang2024irregular} typically couple temporal modules with spatial encoders (e.g., graph convolutions or attention) and achieve strong accuracy when the sensor set and its connectivity are fixed \citep{jin2023spatio}. However, many such approaches are spatially transductive: the spatial component is trained for a specific set of nodes/edges and does not directly support prediction at unseen coordinates without re-parameterization or retraining. Moreover, uncertainty is often quantified via heuristic post-hoc techniques (e.g., ensembles or dropout) rather than through a principled Bayesian treatment. Neural operators learn nonlocal mappings between function spaces and excel in surrogate modeling under dense field supervision \citep{li2020fourier,kovachki2023neural}, but they are not designed for sequential data assimilation from sparse, intermittently missing sensor measurements.

\textbf{Irregular multivariate time series and continuous-time models.}
Irregular multivariate time series (IMTS) forecasting methods \citep{che2018recurrent, rubanova2019latent, de2019gru, scholz2023latent, yalavarthi2024grafiti, zhang2024irregular, liu2026apn} handle missingness and irregular sampling via masking, decay mechanisms, learned time embeddings, or structure-aware aggregation. 
GraFITi \citep{yalavarthi2024grafiti} reformulates irregularly sampled multivariate series with missing values as a sparse bipartite graph and performs prediction via GNN-based edge inference, while APN \citep{liu2026apn} introduces a lightweight patch-based framework with time-aware adaptive patch aggregation to regularize irregular histories before decoding. 
Continuous-time latent dynamics models such as Neural ODEs and CDEs  \citep{chen2018neural,kidger2020neural} naturally accommodate irregular timestamps, and latent ODE/SDE frameworks \citep{rubanova2019latent,li2020scalable} combine these dynamics with variational inference.

\section{Experiments}
\label{sec:experiments}

We evaluate \textsc{NLBST} on synthetic PDE benchmarks and a real-world environmental monitoring dataset.
Our experiments are organized around four research questions:
\begin{compactitem}
    \item \textbf{(RQ1) Forecasting under partial observations:}
    How robust is \textsc{NLBST} to randomly missing measurements at varying rates?

    \item \textbf{(RQ2) Spatially inductive prediction:}
    How well does \textsc{NLBST} generalize to unseen spatial coordinates without retraining?

    \item \textbf{(RQ3) Uncertainty quantification:}
    Are predictive uncertainties well calibrated? %and sharp under observation sparsity?

    \item \textbf{(RQ4) Nonlocal dynamics:}
    Does \textsc{NLBST} improve performance for dynamics dominated by nonlocal coupling, and does performance improve with basis dimension $K$ consistent with Galerkin convergence (Theorem~\ref{thm:galerkin_main})?
\end{compactitem}

% -------------------------------------------------------------------------
\subsection{Datasets}
\label{sec:datasets}

We use two synthetic PDE datasets that provide dense ground-truth fields, along with one real-world environmental monitoring dataset.
For the synthetic datasets, we sample noisy observations at $S$ sensor coordinates and optionally apply missing-at-random masks at rate $\rho_m$.

\textbf{(D1) Advection--diffusion (synthetic).}
We simulate a scalar field $u(x,t)$ over a 2D periodic domain $D=[0,10]^2$ governed by $\frac{\partial u}{\partial t} + \mathbf{v}\cdot \nabla u \;=\; \kappa \,\Delta u + F(x,t)$.

\textbf{(D2) Nonlocal IDE (synthetic).}
We generate fields from the nonlocal IDE
\[
  \partial_t u(x,t)
  = \int_D G_\star(x,x')\,u(x',t)\,dx'
  + \kappa\,\Delta u + f(x,t),
\]
where $G_\star(x,x') = \sum_{p=1}^{R_\star} \lambda_p \psi_p(x)\psi_p(x')$ is a rank-$R_\star$ kernel designed for long-range spatial coupling; we set $R_\star{=}4$. 

\textbf{(D3) EPA PM\textsubscript{2.5} daily (real-world).}
We use daily mean PM\textsubscript{2.5} concentrations from the U.S. Environmental Protection Agency (EPA) Air Quality System (AQS) database.
We select $50$ monitoring stations in the northeastern United States covering $366$ days in 2024, applying a $\log(1{+}y)$ transformation to stabilize variance.
The first 300 days are used for training and the remaining 66 days for evaluation.

% -------------------------------------------------------------------------
\subsection{Evaluation Protocol}
\label{sec:exp_protocol}

\textbf{Rolling-origin forecasting.}
Given a context window length $L$ and forecast horizon $H$, we perform rolling-origin evaluation \citep{tashman2000out} over the test period with stride $r$.
For each window, we run filtering over the context to obtain the terminal belief $(\bm_L^+,\bP_L^+)$, then propagate the latent state forward for $H$ steps without measurement updates. %We report RMSE and CRPS for the evaluation metric. 

\textbf{Missingness setup (RQ1, RQ3).}
On synthetic datasets, we apply i.i.d.\ random missing masks at rates  $\rho_m \in \{10,20,30\}\%$. 
On \textbf{D3}, we use $\rho_m{=}10\%$ in addition to the natural ${\sim}2\%$ missing rate. 

\textbf{Spatial generalization setup (RQ2).}
We designate a fraction of sensor locations as unobserved during training and evaluate predictions at these held-out coordinates without retraining.
On \textbf{D1}, we hold out $\{10,20,30\}\%$ of sensors with $\rho_m{=}0$. 
On \textbf{D3}, we use a $20\%$ holdout. 

\textbf{Nonlocal dynamics setup (RQ4).}
On the nonlocal IDE benchmark (\textbf{D2}), we compare forecasting accuracy under the same rolling-origin protocol and sweep the basis dimension $K\in\{8,12,16,24\}$ to examine how approximation rank affects performance. 

\textbf{Evaluation metrics.} 
We report RMSE and CRPS (Appendix~\ref{app:metrics}). For \textbf{RQ3}, we plot the empirical coverage of $90\%$ prediction intervals and calibration curves. 
We draw $N{=}100$ Monte Carlo samples, and results are averaged over 10 seeds.

% -------------------------------------------------------------------------
\subsection{Baselines}
\label{sec:baselines}

We compare \textsc{NLBST} against baselines spanning discrete-time recurrent models, 
irregular time-series forecasters, 
continuous-time latent dynamics, and neural operators:

\begin{compactitem}
  \item \textbf{Linear-DSTM}~\citep{wikle2010general}\textbf{:}
  A linear state-space model in the same coefficient space as \textsc{NLBST} (shared spatial basis), but with a linear transition matrix replacing the nonlinear Neural ODE dynamics. 

  \item \textbf{GRU-D}~\citep{che2018recurrent}\textbf{:}
  A GRU with trainable exponential decay to handle missing values. 

  \item \textbf{Latent ODE}~\citep{rubanova2019latent}\textbf{:}
  Continuous-time latent dynamics with a VAE-style encoder; uncertainty is obtained via posterior sampling.

  \item \textbf{FNO}~\citep{li2020fourier}\textbf{:}
  A neural operator using spectral convolution along the time axis. %Uncertainty via MC Dropout. 
  \item \textbf{GraFITi}~\citep{yalavarthi2024grafiti}\textbf{:}
  An irregular multivariate time series forecaster that converts partially observed series into a sparse bipartite graph and performs forecasting as an edge-weight prediction task using a graph neural network.
  \item \textbf{APN}~\citep{liu2026apn}\textbf{:}
  A lightweight patch-based model for irregular time series that forms adaptive patches and aggregates them with time-aware patch aggregation. 
\end{compactitem}

\noindent %All baselines share the same train/test splits and rolling-origin evaluation windows. Hyperparameters are tuned on a validation split.
For spatial prediction at unmeasured locations, \textsc{NLBST} and \textsc{Linear-DSTM} use the shared spatial basis expansion.
For the remaining baselines (GRU-D, Latent ODE, FNO, APN, and GraFITi), we apply spatial Gaussian process interpolation to $\cS_u$.
For uncertainty quantification, we use MC Dropout \citep{gal2016dropout} for deterministic neural baselines.

\subsection{Results}
\label{sec:exp_results_overview}

We organize results by research question, integrating findings from the synthetic datasets (\textbf{D1}--\textbf{D2}) and the real-world PM\textsubscript{2.5} dataset (\textbf{D3}).

\textbf{RQ1: Forecasting under partial observations.}
On \textbf{D1}, \textsc{NLBST} achieves the lowest RMSE and CRPS across all missing rates, with stable performance as $\rho_m$ increases from $10\%$ to $30\%$ (Table~\ref{tab:rq1_missingness_main}).
This robustness stems from continuous-time ODE dynamics that bridge missing observations without imputation, combined with Kalman-style updates that restrict conditioning to the observed subset at each time step.
In contrast, FNO degrades sharply with increasing missingness (RMSE $1.65 \to 3.09$), as its spectral convolutions assume a regular temporal grid and are less suited to missing entries. Linear-DSTM exhibits high variance at $\rho_m{=}10\%$ (std $2.05$), with occasional instability in a subset of random seeds.
On \textbf{D3} (Table~\ref{tab:pm25_rq1}), \textsc{NLBST} achieves the lowest RMSE and CRPS at both measured and unmeasured locations under $20\%$ spatial holdout and $10\%$ missingness, improving over Linear-DSTM by ${\sim}8\%$ in RMSE at both measured ($0.48$ vs.\ $0.52$) and unmeasured stations
($0.49$ vs.\ $0.53$). GraFITi is competitive at measured locations (RMSE $0.50$) but degrades at unmeasured ones ($0.58$), indicating limited spatial generalization. The improvement over Linear-DSTM---which shares the same spatial basis and Kalman updates---isolates the contribution of nonlinear continuous-time dynamics via the Neural ODE residual.

\begin{table}[htbp]
\centering
\caption{(\textbf{RQ1}, \textbf{RQ3}) Forecasting under missing-at-random observations on \textbf{D1} (advection--diffusion).
Mean (std) over 10 seeds. Best in \textbf{bold}; second-best \underline{underlined}.}
\label{tab:rq1_missingness_main}
%\setlength{\tabcolsep}{4pt}
%\small
\fontsize{8}{9}\selectfont
\begin{tabular}{lccc}
\toprule
Method & $\rho_m=10\%$ & $\rho_m=20\%$ & $\rho_m=30\%$ \\
\midrule
\multicolumn{4}{l}{\textit{RMSE}} \\[2pt]
\textsc{NLBST} (Ours)
  & \textbf{0.589\,(0.108)} & \textbf{0.620\,(0.087)} & \textbf{0.613\,(0.103)} \\
Linear-DSTM
  & 2.543\,(2.053) & 0.795\,(0.382) & 0.740\,(0.358) \\
GRU-D
  & 1.076\,(0.137) & 0.880\,(0.112) & 0.754\,(0.090) \\
Latent ODE
  & 4.280\,(0.245) & 4.331\,(0.268) & 4.229\,(0.305) \\
FNO
  & 1.651\,(0.803) & 2.746\,(0.581) & 3.089\,(0.551) \\
GraFITi
  & 0.702\,(0.304) & 0.713\,(0.320) & 0.703\,(0.329) \\
APN
  & \underline{0.648\,(0.062)} & \underline{0.655\,(0.056)} & \underline{0.654\,(0.062)} \\
\midrule
\multicolumn{4}{l}{\textit{CRPS}} \\[2pt]
\textsc{NLBST} (Ours)
  & \textbf{0.369\,(0.055)} & \textbf{0.383\,(0.046)} & \textbf{0.384\,(0.051)} \\
Linear-DSTM
  & 1.706\,(1.549) & 0.470\,(0.204) & 0.434\,(0.201) \\
GRU-D
  & 0.743\,(0.118) & 0.569\,(0.089) & 0.473\,(0.069) \\
Latent ODE
  & 3.325\,(0.251) & 3.368\,(0.261) & 3.271\,(0.299) \\
FNO
  & 1.048\,(0.601) & 1.957\,(0.576) & 2.343\,(0.586) \\
GraFITi
  & 0.439\,(0.178) & 0.447\,(0.191) & 0.442\,(0.206) \\
APN
  & \underline{0.424\,(0.147)} & \underline{0.415\,(0.136)} & \underline{0.419\, (0.150)} \\
\bottomrule
\end{tabular}
\end{table}

\begin{table}[htbp]
\centering
\caption{(\textbf{RQ1}, \textbf{RQ2}, \textbf{RQ3}) PM\textsubscript{2.5} daily forecasting (\textbf{D3}): mean (std) over 10 seeds with $20\%$ spatial holdout and $\rho_m{=}10\%$.
Best in \textbf{bold}; second-best \underline{underlined}.}
%\vspace{-0.3em}
\label{tab:pm25_rq1}
\fontsize{8}{9}\selectfont
%\small
\setlength{\tabcolsep}{4pt}
%\resizebox{\columnwidth}{!}{%
\begin{tabular}{l cc cc}
\toprule
& \multicolumn{2}{c}{\textbf{Measured}} & \multicolumn{2}{c}{\textbf{Unmeasured}} \\
\cmidrule(lr){2-3} \cmidrule(lr){4-5}
Model & RMSE & CRPS & RMSE & CRPS \\
\midrule
\textsc{NLBST} (Ours)
  & \textbf{0.48}\,(0.01)
  & \textbf{0.28}\,(0.01)
  & \textbf{0.49}\,(0.02)
  & \textbf{0.29}\,(0.01)
\\
Linear-DSTM
  & 0.52\,(0.03)
  & \underline{0.29}\,(0.01)
  & \underline{0.53}\,(0.02)
  & \underline{0.30}\,(0.01)
\\
GRU-D
  & 1.12\,(0.03)
  & 0.72\,(0.02)
  & 1.03\,(0.04)
  & 0.78\,(0.04)
\\
Latent ODE
  & 1.20\,(0.07)
  & 0.72\,(0.06)
  & 1.09\,(0.07)
  & 0.78\,(0.08)
\\
FNO
  & 1.12\,(0.08)
  & 0.73\,(0.06)
  & 0.90\,(0.09)
  & 0.65\,(0.08)
\\
GraFITi
  & \underline{0.50}\,(0.01)
  & 0.34\,(0.01)
  & 0.58\,(0.11)
  & 0.39\,(0.07)
\\
APN
  & 0.63\,(0.21)
  & 0.44\,(0.19)
  & 0.70\,(0.21)
  & 0.50\,(0.19)
\\
\bottomrule
\end{tabular}
%}
\end{table}

\textbf{RQ2: Spatially inductive prediction.}
On \textbf{D1}, \textsc{NLBST} achieves the lowest RMSE at measured locations across all holdout levels and also the lowest RMSE at held-out,  unmeasured locations (Table~\ref{tab:rq2_D1}), confirming inductive prediction via the spatial basis $\mu(x,t)=\bphi(x)^\top\bz(t)$ without retraining.
On \textbf{D3} (Table~\ref{tab:pm25_rq1}), the unmeasured-location RMSE ($0.49$) is comparable to the measured-location RMSE ($0.48$), indicating that the learned coefficient dynamics generalize to unseen coordinates through basis evaluation alone.
By contrast, the deep learning baselines rely on post-hoc spatial interpolation for held-out stations and incur substantially larger errors at unmeasured locations.

\textbf{RQ3: Uncertainty quantification.}
\textsc{NLBST} achieves the lowest CRPS across various datasets and settings (Tables~\ref{tab:rq1_missingness_main},  \ref{tab:pm25_rq1}, \ref{tab:rq4_baseline}), indicating strong probabilistic forecasting performance. On \textbf{D3}, predictive intervals exhibit sensible time-varying behavior, widening during volatile periods (Figure~\ref{fig:pm25_uq_examples}). 
To assess NLBST’s uncertainty quantification, we construct Gaussian prediction intervals at each nominal level using the posterior mean and standard deviation from Monte Carlo samples, and compute the fraction of test observations contained within each interval as the empirical coverage. Across nominal levels from $10\%$ to $99\%$, the empirical coverage closely aligns with the nominal level (Figure~\ref{fig:pm25_calibration}), demonstrating that the three-level Bayesian hierarchy yields well-calibrated uncertainty.

\begin{table}[htbp]
\caption{(\textbf{RQ2}) Spatially inductive prediction on \textbf{D1} (advection--diffusion).
RMSE at measured (m) and held-out unmeasured (u) locations.
Mean (std) over 10 seeds; \textbf{bold} = best, \underline{underlined} = second best per column.}
%\vspace{-0.3em}
\label{tab:rq2_D1}
\centering\fontsize{8}{9}\selectfont
\setlength{\tabcolsep}{4pt}
\begin{tabular}{l rr rr rr}
\toprule
& \multicolumn{2}{c}{5 held out} & \multicolumn{2}{c}{10 held out} & \multicolumn{2}{c}{15 held out} \\
\cmidrule(lr){2-3}\cmidrule(lr){4-5}\cmidrule(lr){6-7}
Model & m & u & m & u & m & u \\
\midrule
\textsc{NLBST} (Ours) & \textbf{0.53} & \textbf{0.71} & \textbf{0.52} & \textbf{0.78} & \textbf{0.55} & \textbf{0.76} \\
                       & \textit{(0.01)} & \textit{(0.28)} & \textit{(0.04)} & \textit{(0.14)} & \textit{(0.03)} & \textit{(0.12)} \\[2pt]
Linear-DSTM           & 1.04 & 2.39 & 1.27 & 3.12 & 0.84 & 2.02 \\
                       & \textit{(0.35)} & \textit{(0.81)} & \textit{(0.43)} & \textit{(0.62)} & \textit{(0.25)} & \textit{(0.47)} \\[2pt]
GRU-D                 & 1.93 & 2.01 & 1.80 & 1.99 & 1.96 & 2.10 \\
                       & \textit{(0.27)} & \textit{(0.58)} & \textit{(0.13)} & \textit{(0.16)} & \textit{(0.17)} & \textit{(0.27)} \\[2pt]
Latent ODE            & 4.21 & 4.36 & 4.39 & 4.56 & 4.32 & 4.48 \\
                       & \textit{(0.25)} & \textit{(0.22)} & \textit{(0.18)} & \textit{(0.17)} & \textit{(0.08)} & \textit{(0.16)} \\[2pt]
FNO                   & 2.27 & 2.37 & 2.50 & 2.59 & 2.37 & 2.45 \\
                       & \textit{(0.58)} & \textit{(0.23)} & \textit{(0.52)} & \textit{(0.60)} & \textit{(0.43)} & \textit{(0.38)} \\[2pt]
GraFITi               & \underline{0.71} & 1.02 & 0.70 & 0.94 & 0.80 & 1.01 \\
                       & \textit{(0.28)} & \textit{(0.54)} & \textit{(0.28)} & \textit{(0.44)} & \textit{(0.36)} & \textit{(0.46)} \\[2pt]
APN                   & \underline{0.71} & \underline{0.95} & \underline{0.63} & \underline{0.88} & \underline{0.70} & \underline{0.94} \\
                       & \textit{(0.20)} & \textit{(0.40)} & \textit{(0.11)} & \textit{(0.31)} & \textit{(0.11)} & \textit{(0.30)} \\
\bottomrule
\end{tabular}
\end{table}

\begin{table}[htbp]
\caption{(\textbf{RQ3}, \textbf{RQ4}) Forecasting on the nonlocal IDE benchmark (\textbf{D2}). Mean (std) over 10 seeds.}
%\vspace{-0.3em}
\label{tab:rq4_baseline}
\centering%\small
\fontsize{8}{9}\selectfont
%\setlength{\tabcolsep}{3pt}
%\begin{tabular}{lrr}
\begin{tabular*}{0.9\columnwidth}{l @{\extracolsep{\fill}} rr}
\toprule
Model & RMSE & CRPS \\
\midrule
\textsc{NLBST} (Ours) & \textbf{9.54}\,(3.36) & \textbf{6.84}\,(2.26) \\
Linear-DSTM & \underline{14.42}\,(0.95) & \underline{10.10}\,(1.22) \\
GRU-D & 20.15\,(4.72) & 16.79\,(5.05) \\
Latent ODE & 28.31\,(9.27) & 24.04\,(9.55) \\
FNO & 29.60\,(10.69) & 25.25\,(10.62) \\
GraFITi & 15.29\,(4.40) & 10.48\,(3.64) \\
APN & 20.88\,(4.78) & 17.28\,(5.14) \\
\bottomrule
\end{tabular*}
\end{table}
\vspace{1em}

\textbf{RQ4: Nonlocal dynamics.}
On \textbf{D2}, \textsc{NLBST} substantially outperforms all baselines, reducing the RMSE of the best competitor by over $33\%$ (Table~\ref{tab:rq4_baseline}: $9.54$ vs.\ Linear-DSTM $14.42$; CRPS $6.84$ vs.\ $10.10$).
These gains highlight that explicitly learning long-range coupling via $A_{\mathrm{NL}}$ is particularly beneficial when the underlying dynamics are dominated by nonlocal interactions.
Notably, FNO---a neural-operator baseline designed to model nonlocal operator structure--- still performs poorly on \textbf{D2} (RMSE $29.60$), suggesting that capturing nonlocal structure alone is insufficient in this sparse-sensor sequential forecasting setting without principled temporal dynamics.
The basis-dimension sweep (Table~\ref{tab:rq4_ksweep}) shows monotonic improvement as $K$ increases ($11.70 \to 9.54$), consistent with the Galerkin convergence theory (Theorem~\ref{thm:galerkin_main}).

\begin{figure}[htbp]
\centering
\includegraphics[width=0.99\columnwidth]{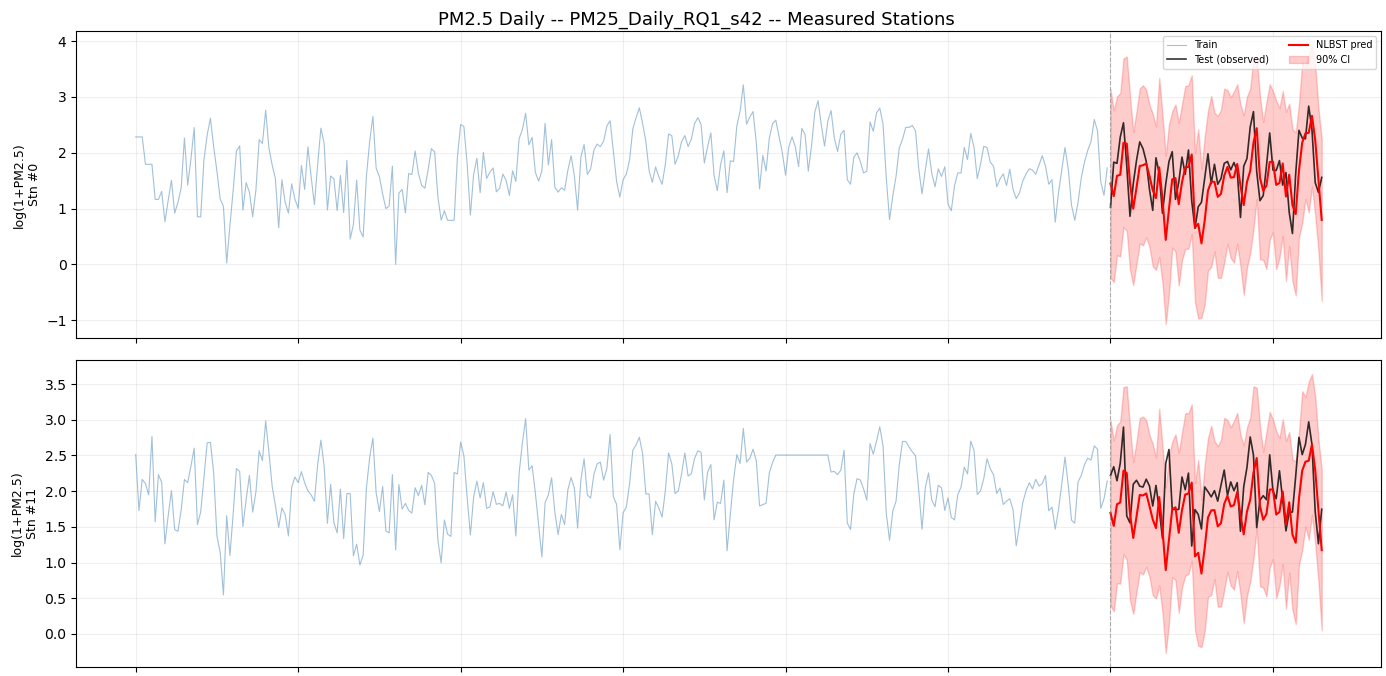}
%\vspace{0.4em}
\includegraphics[width=0.99\columnwidth]{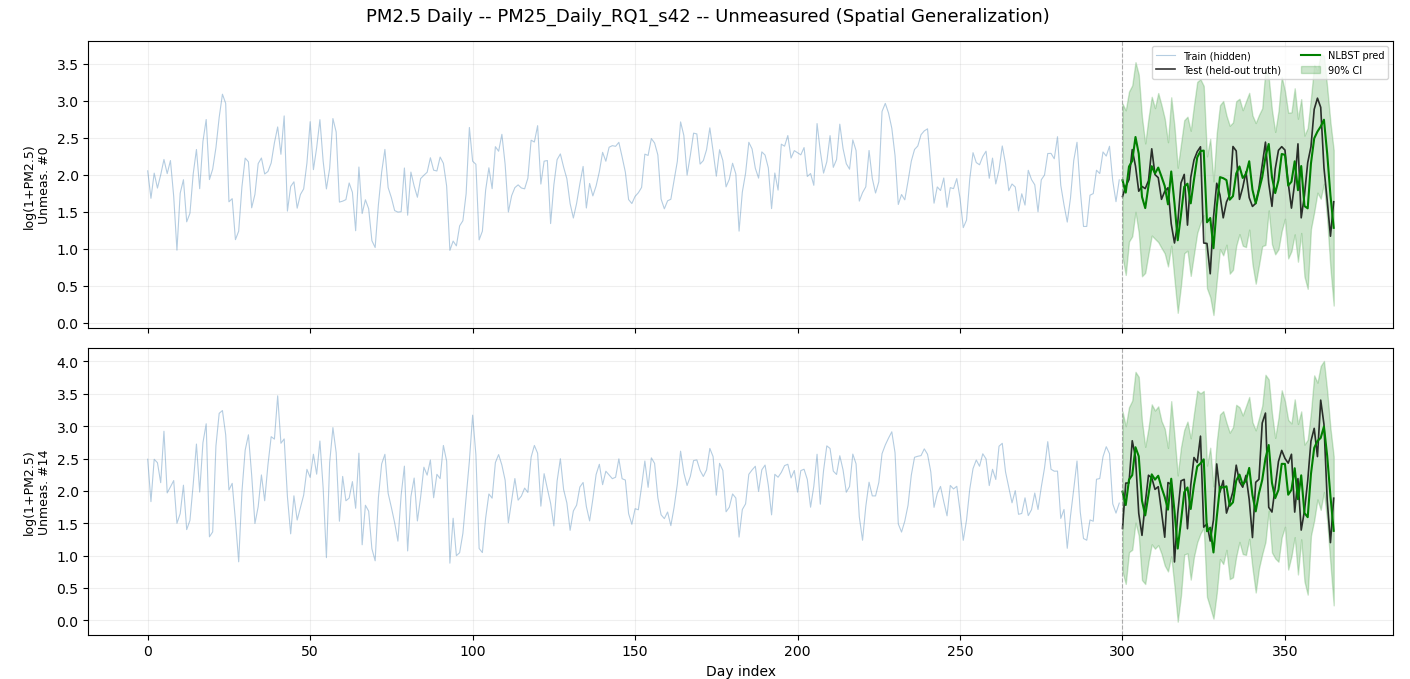}
\caption{(\textbf{RQ3}) EPA PM\textsubscript{2.5} predictive trajectories with $90\%$ intervals.
\textbf{Top:} the upper two subplots (red-shaded), corresponding to measured stations.
\textbf{Bottom:} the lower two subplots (green-shaded), corresponding to unmeasured stations.
Shaded bands denote $90\%$ predictive intervals.}
%\vspace{-0.3em}
\label{fig:pm25_uq_examples}
\end{figure}

\begin{figure}[htbp]
\centering
\includegraphics[width=0.6\columnwidth]{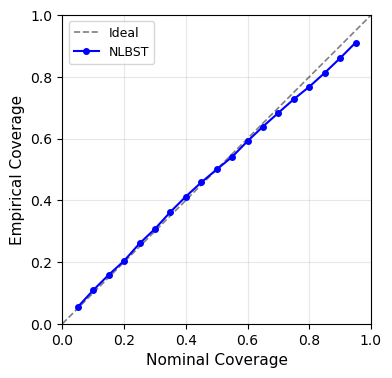}
\caption{(\textbf{RQ3}) Calibration plot for EPA PM\textsubscript{2.5}: nominal vs.\ empirical coverage. The dashed line indicates ideal calibration, and the solid curve (\textsc{NLBST}) tracks it closely.}
\label{fig:pm25_calibration}
\end{figure}

\begin{table}[htbp]
\caption{(\textbf{RQ4}) Basis dimension sweep on \textbf{D2}. RMSE decreases monotonically with $K$, consistent with Galerkin convergence theory (Theorem~\ref{thm:galerkin_main}).}
\label{tab:rq4_ksweep}
\centering%\small
\fontsize{8}{9}\selectfont
\begin{tabular}{lcccc}
\toprule
 & $K{=}8$ & $K{=}12$ & $K{=}16$ & $K{=}24$ \\
\midrule
RMSE & 11.70\,(3.69) & 10.70\,(3.66) & 10.08\,(3.38) & \textbf{9.54}\,(3.36) \\
\bottomrule
\end{tabular}
\end{table}

\textbf{Computational cost and effect of $K$.}
Table~\ref{tab:cost_ksens_main} reports accuracy and per-epoch cost as a function of the
basis dimension $K$ on \textbf{D3}. Accuracy improves up to $K{\approx}24$ and then
saturates (measured RMSE $0.527 \to 0.494$ from $K{=}8$ to $24$, with $<\!1\%$ further
gain up to $K{=}64$), supporting $K{=}24$ as a sensible default. Although the Kalman
update has $\cO(K^3)$ worst-case complexity, per-epoch runtime stays near-constant
($\sim$6.6\,s) and memory grows only mildly ($0.03 \to 0.06$\,GB) as $K$ grows by a factor of eight, indicating that fixed overhead dominates the $\cO(K^3)$ cost within the tested
range. \textsc{NLBST} thus attains strong accuracy at low computational cost. 

\begin{table}[t]
\centering
\caption{(\textbf{$K$-sensitivity \& Computational Cost}) Accuracy and per-epoch cost vs.\ basis dimension
$K$ on \textbf{D3}. Accuracy saturates around $K{=}24$ while
runtime and memory stay nearly constant. RMSE: mean (std) over seeds; cost: mean (std)
over epochs.}
\label{tab:cost_ksens_main}
\fontsize{8}{9}\selectfont
\setlength{\tabcolsep}{4pt}
\begin{tabular}{lcccc}
\toprule
$K$ & RMSE (m) & RMSE (u) & Time/ep.\,(s) & Mem (GB) \\
\midrule
8  & 0.527\,(0.005) & 0.521\,(0.006) & 6.71\,(0.13) & 0.03 \\
16 & 0.499\,(0.004) & 0.485\,(0.013) & 6.55\,(0.03) & 0.03 \\
24 & 0.494\,(0.006) & 0.482\,(0.009) & 6.63\,(0.09) & 0.04 \\
32 & 0.494\,(0.009) & 0.479\,(0.013) & 6.57\,(0.06) & 0.04 \\
64 & 0.490\,(0.010) & 0.480\,(0.018) & 6.58\,(0.05) & 0.06 \\
\bottomrule
\end{tabular}
\end{table}

\section{Conclusion}
\label{sec:conclusion}
We introduced \textsc{NLBST}, a continuous-space, continuous-time hierarchical Bayesian model that learns explicit nonlocal coupling while retaining tractable inference and Kalman-style measurement updates. 
\textsc{NLBST} combines a coordinate-based spatial basis with an ODE-driven latent process and variational learning of global uncertainty parameters, enabling forecasting, inductive prediction at unseen coordinates, and calibrated uncertainty in a unified framework.
By jointly learning nonlocal interactions and continuous-time dynamics within a Bayesian hierarchy, \textsc{NLBST} bridges the gap between expressive neural dynamics and principled uncertainty quantification, which prior methods typically treat in isolation.
Experiments on synthetic dynamics and real-world EPA PM\textsubscript{2.5} data show consistently strong performance, especially in strongly nonlocal and partially observed regimes. The future work includes structured regularization of $A_{\mathrm{NL}}$ and learnable basis-dimension selection.

\begin{acknowledgements}
This work was supported by the National Research Foundation of Korea (NRF) grant funded by the Korea government (MSIT) (2023R1A2C2005453, RS-2023-00218913).
\end{acknowledgements}

% References
\bibliography{bibliography}

@article{lee2020bayesian,
  title={Bayesian nonparametric joint mixture model for clustering spatially correlated time series},
  author={Lee, Youngmin and Kim, Heeyoung},
  journal={Technometrics},
  volume={62},
  number={3},
  pages={313--329},
  year={2020},
  publisher={Taylor \& Francis}
}

@article{jung2024spatially,
  title={Spatially-correlated time series clustering using location-dependent Dirichlet process mixture model},
  author={Jung, Junsub and Kim, Sungil and Kim, Heeyoung},
  journal={Statistical Analysis and Data Mining: The ASA Data Science Journal},
  volume={17},
  number={1},
  pages={e11649},
  year={2024},
  publisher={Wiley Online Library}
}

@article{kim2019spatiotemporal,
  title={Spatiotemporal auto-regressive model for origin--destination air passenger flows},
  author={Kim, Keunseo and Kim, Vinnam and Kim, Heeyoung},
  journal={Journal of the Royal Statistical Society Series A: Statistics in Society},
  volume={182},
  number={3},
  pages={1003--1016},
  year={2019},
  publisher={Oxford University Press}
}

@book{wikle2019spatio,
  title={Spatio-temporal statistics with R},
  author={Wikle, Christopher K and Zammit-Mangion, Andrew and Cressie, Noel},
  year={2019},
  publisher={Chapman and Hall/CRC}
}

@book{cressie2015statistics,
  title={Statistics for spatial data},
  author={Cressie, Noel},
  year={2015},
  publisher={John Wiley \& Sons}
}

@article{rubanova2019latent,
  title={Latent ordinary differential equations for irregularly-sampled time series},
  author={Rubanova, Yulia and Chen, Ricky TQ and Duvenaud, David K},
  journal={Advances in neural information processing systems},
  volume={32},
  year={2019}
}

@article{che2018recurrent,
  title={Recurrent neural networks for multivariate time series with missing values},
  author={Che, Zhengping and Purushotham, Sanjay and Cho, Kyunghyun and Sontag, David and Liu, Yan},
  journal={Scientific reports},
  volume={8},
  number={1},
  pages={6085},
  year={2018},
  publisher={Nature Publishing Group UK London}
}

@article{chen2018neural,
  title={Neural ordinary differential equations},
  author={Chen, Ricky TQ and Rubanova, Yulia and Bettencourt, Jesse and Duvenaud, David K},
  journal={Advances in neural information processing systems},
  volume={31},
  year={2018}
}

@article{kidger2020neural,
  title={Neural controlled differential equations for irregular time series},
  author={Kidger, Patrick and Morrill, James and Foster, James and Lyons, Terry},
  journal={Advances in neural information processing systems},
  volume={33},
  pages={6696--6707},
  year={2020}
}

@inproceedings{gal2016dropout,
  title={Dropout as a bayesian approximation: Representing model uncertainty in deep learning},
  author={Gal, Yarin and Ghahramani, Zoubin},
  booktitle={international conference on machine learning},
  pages={1050--1059},
  year={2016},
  organization={PMLR}
}

@article{clarotto2024spde,
  title={The SPDE approach for spatio-temporal datasets with advection and diffusion},
  author={Clarotto, Lucia and Allard, Denis and Romary, Thomas and Desassis, Nicolas},
  journal={Spatial Statistics},
  volume={62},
  pages={100847},
  year={2024},
  publisher={Elsevier}
}

@article{kumar2024spatio,
  title={Spatio-temporal predictive modeling techniques for different domains: a survey},
  author={Kumar, Rahul and Bhanu, Manish and Mendes-Moreira, Jo{\~a}o and Chandra, Joydeep},
  journal={ACM Computing Surveys},
  volume={57},
  number={2},
  pages={1--42},
  year={2024},
  publisher={ACM New York, NY}
}

@article{faghmous2014spatio,
  title={Spatio-temporal data mining for climate data: Advances, challenges, and opportunities},
  author={Faghmous, James H and Kumar, Vipin},
  journal={Data mining and knowledge discovery for big data: Methodologies, challenge and opportunities},
  pages={83--116},
  year={2014},
  publisher={Springer}
}

@article{jiang2018survey,
  title={A survey on spatial prediction methods},
  author={Jiang, Zhe},
  journal={IEEE transactions on knowledge and Data Engineering},
  volume={31},
  number={9},
  pages={1645--1664},
  year={2018},
  publisher={IEEE}
}

@incollection{pawar2024handling,
  title={Handling Uncertainty in Spatiotemporal Data},
  author={Pawar, Shivaji D and Pawar, Varsha S and Abimannan, Satheesh},
  booktitle={Spatiotemporal Data Analytics and Modeling: Techniques and Applications},
  pages={69--87},
  year={2024},
  publisher={Springer}
}

@article{stroud2001dynamic,
  title={Dynamic models for spatiotemporal data},
  author={Stroud, Jonathan R and M{\"u}ller, Peter and Sans{\'o}, Bruno},
  journal={Journal of the Royal Statistical Society: Series B (Statistical Methodology)},
  volume={63},
  number={4},
  pages={673--689},
  year={2001},
  publisher={Wiley Online Library}
}

@article{wikle2010general,
  title={A general science-based framework for dynamical spatio-temporal models},
  author={Wikle, Christopher K and Hooten, Mevin B},
  journal={Test},
  volume={19},
  number={3},
  pages={417--451},
  year={2010},
  publisher={Springer}
}

@article{zammit2022deep,
  title={Deep compositional spatial models},
  author={Zammit-Mangion, Andrew and Ng, Tin Lok James and Vu, Quan and Filippone, Maurizio},
  journal={Journal of the American Statistical Association},
  volume={117},
  number={540},
  pages={1787--1808},
  year={2022},
  publisher={Taylor \& Francis}
}

@inproceedings{zhang2024irregular,
  title={Irregular traffic time series forecasting based on asynchronous spatio-temporal graph convolutional networks},
  author={Zhang, Weijia and Zhang, Le and Han, Jindong and Liu, Hao and Fu, Yanjie and Zhou, Jingbo and Mei, Yu and Xiong, Hui},
  booktitle={Proceedings of the 30th ACM SIGKDD Conference on Knowledge Discovery and Data Mining},
  pages={4302--4313},
  year={2024}
}

@article{yu2017spatio,
  title={Spatio-temporal graph convolutional networks: A deep learning framework for traffic forecasting},
  author={Yu, Bing and Yin, Haoteng and Zhu, Zhanxing},
  journal={arXiv preprint arXiv:1709.04875},
  year={2017}
}

@article{jin2023spatio,
  title={Spatio-temporal graph neural networks for predictive learning in urban computing: A survey},
  author={Jin, Guangyin and Liang, Yuxuan and Fang, Yuchen and Shao, Zezhi and Huang, Jincai and Zhang, Junbo and Zheng, Yu},
  journal={IEEE transactions on knowledge and data engineering},
  volume={36},
  number={10},
  pages={5388--5408},
  year={2023},
  publisher={IEEE}
}

@inproceedings{yalavarthi2024grafiti,
  title={Grafiti: Graphs for forecasting irregularly sampled time series},
  author={Yalavarthi, Vijaya Krishna and Madhusudhanan, Kiran and Scholz, Randolf and Ahmed, Nourhan and Burchert, Johannes and Jawed, Shayan and Born, Stefan and Schmidt-Thieme, Lars},
  booktitle={Proceedings of the AAAI Conference on Artificial Intelligence},
  volume={38},
  number={15},
  pages={16255--16263},
  year={2024}
}

@article{koo2024deep,
  title={Deep Latent Factor Model for Spatio-Temporal Forecasting},
  author={Koo, Wonmo and Ma, Eun-Yeol and Kim, Heeyoung},
  journal={Technometrics},
  volume={66},
  number={3},
  pages={470--482},
  year={2024},
  publisher={Taylor \& Francis}
}

@article{wu2024continuously,
  title={Continuously evolving graph neural controlled differential equations for traffic forecasting},
  author={Wu, Jiajia and Chen, Ling},
  journal={arXiv preprint arXiv:2401.14695},
  year={2024}
}

@inproceedings{choi2022graph,
  title={Graph neural controlled differential equations for traffic forecasting},
  author={Choi, Jeongwhan and Choi, Hwangyong and Hwang, Jeehyun and Park, Noseong},
  booktitle={Proceedings of the AAAI conference on artificial intelligence},
  volume={36},
  number={6},
  pages={6367--6374},
  year={2022}
}

@article{tibau2022spatiotemporal,
  title={A spatiotemporal stochastic climate model for benchmarking causal discovery methods for teleconnections},
  author={Tibau, Xavier-Andoni and Reimers, Christian and Gerhardus, Andreas and Denzler, Joachim and Eyring, Veronika and Runge, Jakob},
  journal={Environmental Data Science},
  volume={1},
  pages={e12},
  year={2022}
}

@article{wallace1981teleconnections,
  title={Teleconnections in the geopotential height field during the Northern Hemisphere winter},
  author={Wallace, John M and Gutzler, David S},
  journal={Monthly weather review},
  volume={109},
  number={4},
  pages={784--812},
  year={1981}
}

@article{yeh2018enso,
  title={ENSO atmospheric teleconnections and their response to greenhouse gas forcing},
  author={Yeh, Sang-Wook and Cai, Wenju and Min, Seung-Ki and McPhaden, Michael J and Dommenget, Dietmar and Dewitte, Boris and Collins, Matthew and Ashok, Karumuri and An, Soon-Il and Yim, Bo-Young and others},
  journal={Reviews of Geophysics},
  volume={56},
  number={1},
  pages={185--206},
  year={2018},
  publisher={Wiley Online Library}
}

@article{alizadeh2024review,
  title={A review of ENSO teleconnections at present and under future global warming},
  author={Alizadeh, Omid},
  journal={Wiley Interdisciplinary Reviews: Climate Change},
  volume={15},
  number={1},
  pages={e861},
  year={2024},
  publisher={Wiley Online Library}
}

@article{yang2018nino,
  title={El Ni{\~n}o--Southern Oscillation and its impact in the changing climate},
  author={Yang, Song and Li, Zhenning and Yu, Jin-Yi and Hu, Xiaoming and Dong, Wenjie and He, Shan},
  journal={National Science Review},
  volume={5},
  number={6},
  pages={840--857},
  year={2018},
  publisher={Oxford University Press}
}

@article{wikle1998hierarchical,
  title={Hierarchical Bayesian space-time models},
  author={Wikle, Christopher K and Berliner, L Mark and Cressie, Noel},
  journal={Environmental and ecological statistics},
  volume={5},
  number={2},
  pages={117--154},
  year={1998},
  publisher={Springer}
}

@article{stuart2010inverse,
  title={Inverse problems: a Bayesian perspective},
  author={Stuart, Andrew M},
  journal={Acta numerica},
  volume={19},
  pages={451--559},
  year={2010},
  publisher={Cambridge University Press}
}

@incollection{dashti2015bayesian,
  title={The Bayesian approach to inverse problems},
  author={Dashti, Masoumeh and Stuart, Andrew M},
  booktitle={Handbook of uncertainty quantification},
  pages={1--118},
  year={2015},
  publisher={Springer}
}

@article{laing2003pde,
  title={PDE methods for nonlocal models},
  author={Laing, Carlo R and Troy, William C},
  journal={SIAM Journal on Applied Dynamical Systems},
  volume={2},
  number={3},
  pages={487--516},
  year={2003},
  publisher={SIAM}
}

@article{li2020fourier,
  title={Fourier neural operator for parametric partial differential equations},
  author={Li, Zongyi and Kovachki, Nikola and Azizzadenesheli, Kamyar and Liu, Burigede and Bhattacharya, Kaushik and Stuart, Andrew and Anandkumar, Anima},
  journal={arXiv preprint arXiv:2010.08895},
  year={2020}
}

@article{tashman2000out,
  title={Out-of-sample tests of forecasting accuracy: an analysis and review},
  author={Tashman, Leonard J},
  journal={International journal of forecasting},
  volume={16},
  number={4},
  pages={437--450},
  year={2000},
  publisher={Elsevier}
}

@article{sigrist2015stochastic,
  title={Stochastic partial differential equation based modelling of large space--time data sets},
  author={Sigrist, Fabio and K{\"u}nsch, Hans R and Stahel, Werner A},
  journal={Journal of the Royal Statistical Society Series B: Statistical Methodology},
  volume={77},
  number={1},
  pages={3--33},
  year={2015},
  publisher={Oxford University Press}
}

@article{lindgren2011explicit,
  title={An explicit link between Gaussian fields and Gaussian Markov random fields: the stochastic partial differential equation approach},
  author={Lindgren, Finn and Rue, H{\aa}vard and Lindstr{\"o}m, Johan},
  journal={Journal of the Royal Statistical Society Series B: Statistical Methodology},
  volume={73},
  number={4},
  pages={423--498},
  year={2011},
  publisher={Oxford University Press}
}

@inproceedings{liu2026apn,
 title     = {Rethinking Irregular Time Series Forecasting: A Simple yet Effective Baseline},
 author    = {Xvyuan Liu and Xiangfei Qiu and Xingjian Wu and Zhengyu Li and Chenjuan Guo and Jilin Hu and Bin Yang},
 booktitle = {AAAI},
 year      = {2026}
}

@article{kovachki2023neural,
  title={Neural operator: Learning maps between function spaces with applications to pdes},
  author={Kovachki, Nikola and Li, Zongyi and Liu, Burigede and Azizzadenesheli, Kamyar and Bhattacharya, Kaushik and Stuart, Andrew and Anandkumar, Anima},
  journal={Journal of Machine Learning Research},
  volume={24},
  number={89},
  pages={1--97},
  year={2023}
}

@inproceedings{li2020scalable,
  title={Scalable gradients for stochastic differential equations},
  author={Li, Xuechen and Wong, Ting-Kam Leonard and Chen, Ricky TQ and Duvenaud, David},
  booktitle={International conference on artificial intelligence and statistics},
  pages={3870--3882},
  year={2020},
  organization={PMLR}
}

@article{de2019gru,
  title={GRU-ODE-Bayes: Continuous modeling of sporadically-observed time series},
  author={De Brouwer, Edward and Simm, Jaak and Arany, Adam and Moreau, Yves},
  journal={Advances in neural information processing systems},
  volume={32},
  year={2019}
}

@article{scholz2023latent,
  title={Latent Linear ODEs with Neural Kalman Filtering for Irregular Time Series Forecasting},
  author={Scholz, Randolf and Born, Stefan and Duong-Trung, Nghia and Cruz-Bournazou, Mariano Nicolas and Schmidt-Thieme, Lars},
  year={2023}
}

@inproceedings{liu2022msdr,
  title={Msdr: Multi-step dependency relation networks for spatial temporal forecasting},
  author={Liu, Dachuan and Wang, Jin and Shang, Shuo and Han, Peng},
  booktitle={Proceedings of the 28th ACM SIGKDD conference on knowledge discovery and data mining},
  pages={1042--1050},
  year={2022}
}

@inproceedings{guo2019attention,
  title={Attention based spatial-temporal graph convolutional networks for traffic flow forecasting},
  author={Guo, Shengnan and Lin, Youfang and Feng, Ning and Song, Chao and Wan, Huaiyu},
  booktitle={Proceedings of the AAAI conference on artificial intelligence},
  volume={33},
  number={01},
  pages={922--929},
  year={2019}
}

\newpage

\onecolumn

\makeatletter
\renewcommand{\thanks}[1]{\textsuperscript{*}} 
\global\let\@thanks\@empty    
\makeatother

\title{Nonlocal Bayesian Modeling of Continuous Spatio-Temporal Dynamics\\(Supplementary Material)}
\maketitle

%This Supplementary Material should be submitted together with the main paper.

\appendix

\section{Theoretical Analysis}
\label{app:theory}
This section provides the detailed analysis and proofs for the theoretical guarantees stated in Sec.~\ref{sec:theory}, (Theorems~\ref{thm:wellposed_main}--\ref{thm:galerkin_main}).

\subsection{Problem Setup}
\label{app:theory:setup}

\paragraph{Nonlocal stochastic evolution equations.}
Let $\cH=L^2(D)$ with norm $\|\cdot\|$ and inner product $\langle\cdot,\cdot\rangle$.
A nonlocal kernel $G(x,x')$ induces the integral operator $\cG:\cH\to\cH$,
\[
(\cG u)(x) \;=\; \int_D G(x,x')\,u(x')\,dx'.
\]
We consider the $\cH$-valued nonlocal stochastic evolution of the latent field $\mu(t)\in\cH$:
\begin{equation}
\label{eq:spde_app}
d\mu(t)
= \big(\cG\mu(t) + \cF_\theta(t,\mu(t))\big)\,dt
+ \Sigma(t,\mu(t))\,dW_t,
\end{equation}
where $W$ is a $Q$-Wiener process on $\cH$.
The drift and diffusion are defined by applying time-dependent scalar-valued functions
$f_\theta(t,\cdot):\R\to\R$ and $\sigma(t,\cdot):\R\to\R$ pointwise to the field:
\[
(\cF_\theta(t,u))(x)=f_\theta(t,u(x)),
\qquad
(\Sigma(t,u)v)(x)=\sigma(t,u(x))\,v(x),
\quad u,v\in\cH.
\]

\paragraph{Galerkin projection.}
Let $\{\phi_i\}_{i\ge 1}\subset\cH$ and $V_K=\mathrm{span}\{\phi_1,\dots,\phi_K\}$.
For simplicity, let $P_K:\cH\to V_K$ denote the orthogonal projection (non-orthonormal bases can be handled via the Gram/mass matrix).
The Galerkin projected dynamics on $V_K$ are defined by
\begin{equation}
\label{eq:galerkin_projected}
d\mu_K(t) = P_K\!\big(\cG\mu_K(t) + \cF_\theta(t,\mu_K(t))\big)\,dt + P_K\Sigma(t,\mu_K(t))\,dW_t,
\quad \mu_K(0)=P_K\mu_0.
\end{equation}

\subsection{Preliminaries}
\label{app:theory:prelim}

\paragraph{Semigroups of linear operators.} Let $H$ be a Hilbert space. 
\begin{definition}[Semigroup] A family $\{S(t)\}_{t \ge 0}$ of bounded linear operators $S(t):H \to H$ is a \emph{semigroup} if 
\begin{compactenum}
    \item $S(0)=I$, 
    \item $S(t+s)=S(t)S(s)$ for all $s,t \ge 0$. 
\end{compactenum} 
\end{definition} 
\begin{definition}[Infinitesimal generator] The \emph{infinitesimal generator} of a semigroup $\{S(t)\}_{t \ge 0}$ is the operator $A$ defined by \[ Au \;:=\; \lim_{h \downarrow 0}\frac{S(h)u - u}{h}, \] on the domain $D(A)=\{u\in H: \text{the limit exists in }H\}$. \end{definition} If $A$ is a bounded operator on $H$, then $S(t)=e^{tA}$ is a \emph{uniformly continuous} semigroup satisfying 
\begin{equation} 
\label{eq:semigroup_bound_app} 
\|S(t)\|_{\mathcal{L}(H)} \;\le\; e^{t\|A\|}. 
\end{equation}

\paragraph{Mild solutions (variation of constants).}
Consider a semilinear stochastic evolution on a Hilbert space $H$ of the form
\begin{equation}
\label{eq:semilinear_stoch}
du(t) = \big(Au(t) + F(t,u(t))\big)\,dt + \Sigma(t,u(t))\,dW_t,
\qquad u(0)=u_0\in H,
\end{equation}
where $A$ generates a semigroup $S(t)$ on $H$ and $W$ is a $Q$-Wiener process.
A process $u$ is called a \emph{mild solution} if it satisfies the variation-of-constants formula
\begin{equation}
\label{eq:mild_form_app}
u(t) = S(t)\,u_0
+ \int_0^t S(t-s)\,F(s,u(s))\,ds
+ \int_0^t S(t-s)\,\Sigma(s,u(s))\,dW_s.
\end{equation}
In our setting Eq.\eqref{eq:spde_app}, we take $H=\cH$, $A=\cG$, and $F=\cF_\theta$.

\paragraph{Banach fixed-point theorem.}
Let $(X,d)$ be a complete metric space and $J:X\to X$ satisfy $d(Ju,Jv)\le \gamma\,d(u,v)$ for some $\gamma\in[0,1)$.
Then $J$ has a unique fixed point in $X$.

\paragraph{Gr\"onwall's inequality.}
If $z(t)\le \alpha + \beta\int_0^t z(s)\,ds$ with $\alpha,\beta\ge 0$, then $z(t)\le \alpha e^{\beta t}$ for all $t\in[0,T]$.

\paragraph{Burkholder-Davis-Gundy inequality (BDG inequality).}
For a predictable process $B(t)$ with $B(t)Q^{1/2}$ Hilbert--Schmidt,
\begin{equation}
%\[
\E\Big[\sup_{t\in[0,T]}\Big\|\int_0^t B(s)\,dW_s\Big\|^2\Big]
\le C_{\mathrm{BDG}}\,
\E\int_0^T \|B(s)Q^{1/2}\|_{\mathrm{HS}}^2\,ds.
%\]
\end{equation}

\subsection{Assumptions}
\label{app:theory:assum}

\begin{assumption}[Bounded nonlocal operator]
\label{ass:bounded}
The kernel $G$ is square-integrable on $D\times D$:
\[
\int_D\int_D |G(x,x')|^2\,dx\,dx' < \infty,
\]
so $\cG$ is Hilbert--Schmidt (hence bounded) on $\cH$.
\end{assumption}

\begin{assumption}[Lipschitz and linear growth]
\label{ass:lipschitz}
There exist constants $L,C>0$ such that for all $t\in[0,T]$ and $u,v\in\R$,
\begin{align*}
|f_\theta(t,u)-f_\theta(t,v)| &\le L|u-v|,\\
|\sigma(t,u)-\sigma(t,v)| &\le L|u-v|,\\
|f_\theta(t,u)|^2 + |\sigma(t,u)|^2 &\le C(1+|u|^2).
\end{align*}
\end{assumption}

\begin{assumption}[$Q$-Wiener noise regularity]
\label{ass:q_wiener}
Let $W$ be a $Q$-Wiener process on $\cH$ with $\mathrm{Tr}(Q)<\infty$.
Assume $\Sigma(t,u)Q^{1/2}$ is Hilbert--Schmidt for all $t\in[0,T]$ and $u\in\cH$, and there exist
$L_\Sigma,C_\Sigma>0$ such that for all $u,v\in\cH$,
\[
\|\Sigma(t,u)Q^{1/2}-\Sigma(t,v)Q^{1/2}\|_{\mathrm{HS}} \le L_\Sigma \|u-v\|,
\qquad
\|\Sigma(t,u)Q^{1/2}\|_{\mathrm{HS}}^2 \le C_\Sigma(1+\|u\|^2).
\]
\end{assumption}

\begin{assumption}[Initial condition]
\label{ass:initial}
The initial condition $\mu_0\in L^2(\Omega;\cH)$ is independent of $W$ and satisfies
$\E[\|\mu_0\|^2]<\infty$.
\end{assumption}

\begin{assumption}[Approximation property] 
\label{ass:approx} 
$\overline{\bigcup_{K\ge 1} V_K}=\cH$. 
\end{assumption}

\subsection{Proof of Theorem~\ref{thm:wellposed_main} (Well-posedness)}

\begin{proof}
Let $S(t)=e^{t\cG}$. Since $\cG$ is bounded (Assumption~\ref{ass:bounded}),
Eq.\eqref{eq:semigroup_bound_app} with $A=\cG$ implies
\[
\|S(t)\|_{\mathcal{L}(\cH)} \le e^{t\|\cG\|} \le M_T:=e^{T\|\cG\|}, \qquad t\in[0,T].
\]
We first note that the pointwise drift map is Lipschitz on $\cH=L^2(D)$.
By Assumption~\ref{ass:lipschitz}, for all $t\in[0,T]$ and $u,v\in L^2(D)$,
\begin{align*}
\|\cF_\theta(t,u)-\cF_\theta(t,v)\|_{L^2(D)}^2
&= \int_D \big|f_\theta(t,u(x)) - f_\theta(t,v(x))\big|^2\,dx \\
&\le \int_D \big(L|u(x)-v(x)|\big)^2\,dx \\
&= L^2 \|u-v\|_{L^2(D)}^2,
\end{align*}
and hence $\|\cF_\theta(t,u)-\cF_\theta(t,v)\|_{L^2(D)} \le L\|u-v\|_{L^2(D)}$.
Similarly, Assumption~\ref{ass:q_wiener} gives for $u,v\in\cH$,
\[
\|\big(\Sigma(t,u)-\Sigma(t,v)\big)Q^{1/2}\|_{\mathrm{HS}}
\le L_\Sigma\|u-v\|.
\]

\paragraph{Existence and uniqueness on a short interval.}
Fix a horizon $\tau\in(0,T]$ and consider the Banach space
$\mathcal{X}_\tau := L^2(\Omega;C([0,\tau];\cH))$ with norm
$\|u\|_{\mathcal{X}_\tau}^2 := \E\big[\sup_{t\in[0,\tau]}\|u(t)\|^2\big]$.
Define the mild-map $\mathcal{T}$ by
\[
(\mathcal{T}u)(t):=S(t)\mu_0+\int_0^t S(t-s)\cF_\theta(s,u(s))\,ds
+\int_0^t S(t-s)\Sigma(s,u(s))\,dW_s .
\]
A fixed point of $\mathcal{T}$ is a mild solution (cf.\ Eq.\eqref{eq:mild_form_app} with $A=\cG$ and $F=\cF_\theta$).

We use the BDG inequality in $\cH$:
for predictable $B(t)$ with $B(t)Q^{1/2}$ Hilbert--Schmidt,
\[
\E\Big[\sup_{t\in[0,\tau]}\Big\|\int_0^t B(s)\,dW_s\Big\|^2\Big]
\le C_{\mathrm{BDG}}\E\int_0^\tau \|B(s)Q^{1/2}\|_{\mathrm{HS}}^2\,ds.
\]
For $u,v\in\mathcal{X}_\tau$, let $\Delta F(s):=\cF_\theta(s,u(s))-\cF_\theta(s,v(s))$.
Then
\begin{align*}
&\E\Bigg[\sup_{t\in[0,\tau]}\Big\|\int_0^t S(t-s)\Delta F(s)\,ds\Big\|^2\Bigg]\\
&\le
\E\Bigg[\sup_{t\in[0,\tau]}\Big(\int_0^t \|S(t-s)\Delta F(s)\|\,ds\Big)^2\Bigg]
\qquad \text{(Minkowski: $\|\int g\|\le \int \|g\|$)}\\
&\le
\E\Bigg[\sup_{t\in[0,\tau]}\Big(\int_0^t \|S(t-s)\|_{\mathcal{L}(\cH)}\,\|\Delta F(s)\|\,ds\Big)^2\Bigg]
\qquad \text{(sub-multiplicativity)}\\
&\le
\E\Bigg[\sup_{t\in[0,\tau]}\Big(\int_0^t M_T\,\|\Delta F(s)\|\,ds\Big)^2\Bigg]
\qquad \text{(semigroup bound)}\\
&=
M_T^2\,\E\Bigg[\sup_{t\in[0,\tau]}\Big(\int_0^t \|\Delta F(s)\|\,ds\Big)^2\Bigg]\\
&\le
M_T^2\,\E\Bigg[\sup_{t\in[0,\tau]}\Big( t\int_0^t \|\Delta F(s)\|^2\,ds \Big)\Bigg]
\qquad \text{(Cauchy--Schwarz: $(\int_0^t a)^2\le t\int_0^t a^2$)}\\
&\le
M_T^2\,\tau \int_0^\tau \E\|\Delta F(s)\|^2\,ds\\
&\le
M_T^2\,\tau \int_0^\tau L^2\,\E\|u(s)-v(s)\|^2\,ds
\qquad \text{(Lipschitz of $\cF_\theta$)}\\
&\le
M_T^2\,L^2\,\tau^2\,\E\Big[\sup_{s\in[0,\tau]}\|u(s)-v(s)\|^2\Big]
\qquad \text{(since $\int_0^\tau f(s)\,ds \le \tau \sup_s f(s)$)}\\
&=
M_T^2\,L^2\,\tau^2\,\|u-v\|_{\mathcal{X}_\tau}^2.
\qquad \text{(definition of $\|\cdot\|_{\mathcal{X}_\tau}$)}
\end{align*}
For the stochastic term, BDG and Assumption~\ref{ass:q_wiener} yield
\begin{align*}
&\E\Big[\sup_{t\in[0,\tau]}\Big\|\int_0^t S(t-s)\big(\Sigma(s,u(s))-\Sigma(s,v(s))\big)\,dW_s\Big\|^2\Big]\\
&\qquad \le C_{\mathrm{BDG}} M_T^2 \int_0^\tau \E\|\big(\Sigma(s,u(s))-\Sigma(s,v(s))\big)Q^{1/2}\|_{\mathrm{HS}}^2\,ds
\;\le\; C_{\mathrm{BDG}} M_T^2 L_\Sigma^2 \tau \|u-v\|_{\mathcal{X}_\tau}^2.
\end{align*}
Combining the two bounds gives
\[
\|\mathcal{T}u-\mathcal{T}v\|_{\mathcal{X}_\tau}^2
\le \Big( M_T^2 L^2 \tau^2 + C_{\mathrm{BDG}} M_T^2 L_\Sigma^2 \tau \Big)\|u-v\|_{\mathcal{X}_\tau}^2.
\]
Hence for $\tau>0$ sufficiently small, $\mathcal{T}$ is a contraction on $\mathcal{X}_\tau$.
By Banach's fixed-point theorem, there exists a unique mild solution on $[0,\tau]$.
Repeating the argument on successive intervals yields a unique mild solution on $[0,T]$.
\end{proof}

\subsection{Proof of Theorem~\ref{thm:stability_main} (Stability)}
\begin{proof}
Let $\mu,\tilde\mu$ be mild solutions with initial conditions $\mu_0,\tilde\mu_0$.
From the mild form and $S(t)=e^{t\cG}$, for any $t\in[0,T]$ we have
\begin{align*}
\mu(t)-\tilde\mu(t)
&= S(t)(\mu_0-\tilde\mu_0)
+ \int_0^t S(t-s)\big(\cF_\theta(s,\mu(s))-\cF_\theta(s,\tilde\mu(s))\big)\,ds\\
&\quad + \int_0^t S(t-s)\big(\Sigma(s,\mu(s))-\Sigma(s,\tilde\mu(s))\big)\,dW_s.
\end{align*}
Let $\Delta(t):=\mu(t)-\tilde\mu(t)$, $\Delta F(s):=\cF_\theta(s,\mu(s))-\cF_\theta(s,\tilde\mu(s))$, and $\Delta\Sigma(s):=\Sigma(s,\mu(s))-\Sigma(s,\tilde\mu(s))$. \\
For all $t\in[0,T]$,
\begin{align*}
z(t)
&:= \E\Big[\sup_{r\in[0,t]}\|\mu(r)-\tilde\mu(r)\|^2\Big]\\
&\le 3\,\E\Big[\sup_{r\in[0,t]}\|S(r)(\mu_0-\tilde\mu_0)\|^2\Big]
 +3\,\E\Big[\sup_{r\in[0,t]}\Big\|\int_0^r S(r-s)\Delta F(s)\,ds\Big\|^2\Big]\\
&\quad +3\,\E\Big[\sup_{r\in[0,t]}\Big\|\int_0^r S(r-s)\Delta\Sigma(s)\,dW_s\Big\|^2\Big]
\qquad \text{(mild form, $(a{+}b{+}c)^2\le 3(a^2{+}b^2{+}c^2)$)}\\[2pt]
&\le 3 M_T^2\,\E\|\mu_0-\tilde\mu_0\|^2
 +3 M_T^2\, t \int_0^t \E\|\Delta F(s)\|^2\,ds
 +3 C_{\mathrm{BDG}} M_T^2 \int_0^t \E\|\Delta\Sigma(s)Q^{1/2}\|_{\mathrm{HS}}^2\,ds \\
&\hspace{8.2cm}\text{(semigroup bound, Minkowski, BDG)}\\
&\le 3 M_T^2\,\E\|\mu_0-\tilde\mu_0\|^2
 +3 M_T^2\, L^2 t \int_0^t \E\|\mu(s)-\tilde\mu(s)\|^2\,ds
  +3 C_{\mathrm{BDG}} M_T^2\, L_\Sigma^2 \int_0^t \E\|\mu(s)-\tilde\mu(s)\|^2\,ds \\
&\hspace{8.2cm} \text{(Lipschitz)}\\
&\le 3 M_T^2\,\E\|\mu_0-\tilde\mu_0\|^2
 +3 M_T^2\big(L^2 t + C_{\mathrm{BDG}}L_\Sigma^2\big)\int_0^t z(s)\,ds
\qquad \text{(since $\|\Delta(s)\|^2\le \sup_{\rho\le s}\|\Delta(\rho)\|^2$)}\\
&\le 3 M_T^2\,\E\|\mu_0-\tilde\mu_0\|^2
 +3 M_T^2\big(L^2 T + C_{\mathrm{BDG}}L_\Sigma^2\big)\int_0^t z(s)\,ds.
\end{align*}
Let $C_0:=3M_T^2$ and $C_1:=3M_T^2\big(L^2T+C_{\mathrm{BDG}}L_\Sigma^2\big)$.\\ Then $z(t)\le C_0\,\E\|\mu_0-\tilde\mu_0\|^2 + C_1\int_0^t z(s)\,ds$,
so Gr\"onwall's inequality yields
\[
\E\Big[\sup_{t\in[0,T]}\|\mu(t)-\tilde\mu(t)\|^2\Big] = z(T)\le C_0 e^{C_1T}\E\|\mu_0-\tilde\mu_0\|^2.
\]
This implies the claimed stability bound ($C_T=C_0 e^{C_1T}$).
\end{proof}

\subsection{Proof of Theorem~\ref{thm:galerkin_main} (Galerkin Convergence)}

\begin{proof}
Let $P_K:\cH\to V_K$ be the orthogonal projection.
Since the Galerkin dynamics evolve in the finite-dimensional space $V_K$ and the projected drift/diffusion inherit the same Lipschitz and growth bounds, Eq.\eqref{eq:galerkin_projected} admits a unique solution $\mu_K$ (by the same argument as Theorem~\ref{thm:wellposed_main}).

Define the error $e_K(t):=\mu(t)-\mu_K(t)$.
Because $\cG$ is bounded, the mild formulation is equivalent to the integral form of Eq.\eqref{eq:spde_app}:
\[
\mu(t)=\mu_0+\int_0^t \big(\cG\mu(s)+\cF_\theta(s,\mu(s))\big)\,ds+\int_0^t \Sigma(s,\mu(s))\,dW_s,
\]
and similarly from Eq.\eqref{eq:galerkin_projected},
\[
\mu_K(t)=P_K\mu_0+\int_0^t P_K\big(\cG\mu_K(s)+\cF_\theta(s,\mu_K(s))\big)\,ds+\int_0^t P_K\Sigma(s,\mu_K(s))\,dW_s.
\]
Subtracting gives
\begin{align*}
e_K(t)
&=(I-P_K)\mu_0
+\int_0^t (I-P_K)\big(\cG\mu(s)+\cF_\theta(s,\mu(s))\big)\,ds
+\int_0^t (I-P_K)\Sigma(s,\mu(s))\,dW_s\\
&\quad+\int_0^t P_K\Big(\cG e_K(s) + \big(\cF_\theta(s,\mu(s))-\cF_\theta(s,\mu_K(s))\big)\Big)\,ds
+\int_0^t P_K\big(\Sigma(s,\mu(s))-\Sigma(s,\mu_K(s))\big)\,dW_s.
\end{align*}

Let $Z_K(t):=\E\big[\sup_{r\in[0,t]}\|e_K(r)\|^2\big]$.
Using Cauchy--Schwarz for the deterministic integrals, the BDG inequality for the stochastic integrals, the boundedness of $\cG$,
and the Lipschitz bounds on $\cF_\theta$ and $\Sigma(\cdot)Q^{1/2}$, we obtain for all $t\in[0,T]$,
\begin{equation}
\label{eq:galerkin_gronwall_refined}
Z_K(t)\;\le\;\alpha_K(t)\;+\;\beta \int_0^t Z_K(s)\,ds,
\end{equation}
where
\begin{align*}
\alpha_K(t)
&:= C_0\Bigg(
\E\|(I-P_K)\mu_0\|^2
+ t \int_0^t \E\|(I-P_K)\cG\mu(s)\|^2\,ds
+ t \int_0^t \E\|(I-P_K)\cF_\theta(s,\mu(s))\|^2\,ds \\
&\hspace{2.6cm}
+ C_{\mathrm{BDG}}\int_0^t \E\|(I-P_K)\Sigma(s,\mu(s))Q^{1/2}\|_{\mathrm{HS}}^2\,ds
\Bigg),
\\[2pt]
\beta
&:= C_1\Big( t\|\cG\|^2 + tL^2 + C_{\mathrm{BDG}}L_\Sigma^2 \Big)
\;\le\; C_1\Big( T\|\cG\|^2 + TL^2 + C_{\mathrm{BDG}}L_\Sigma^2 \Big),
\end{align*}
for constants $C_0,C_1>0$ independent of $K$ (depending only on norm inequalities).
Applying Gr\"onwall's inequality to Eq.\eqref{eq:galerkin_gronwall_refined} yields
\[
Z_K(T)\le \alpha_K(T)\exp(\beta T).
\]
Since $P_K\to I$ strongly on $\cH$ (Assumption~\ref{ass:approx}) and $\mu$ has finite second moments on $[0,T]$
(Theorem~\ref{thm:wellposed_main} and linear growth), we have $\alpha_K(T)\to 0$ as $K\to\infty$ by dominated convergence theorem.
Therefore $Z_K(T)\to 0$, i.e.,
\[
\E\Big[\sup_{t\in[0,T]}\|\mu(t)-\mu_K(t)\|^2\Big]\to 0\quad \text{as }K\to\infty.
\]
\end{proof}

\section{ELBO Details}
\label{app:elbo_details}

The nonlocal kernel matrix $A_{\mathrm{NL}}$, Neural ODE parameters $\theta$, and basis parameters are optimized by maximizing an evidence lower bound (ELBO) via stochastic variational inference (SVI).
Let the variational family factorize as
$
q(\sigma_{\mathrm{obs}},\sigma_{\mathrm{proc}},\bz_{1:T})
=
q(\sigma_{\mathrm{obs}})\,q(\sigma_{\mathrm{proc}})
\prod_{k=1}^{T} q(\bz_k \mid \cD_{1:k}, \sigma_{\mathrm{obs}}, \sigma_{\mathrm{proc}}),
$
where each $q(\bz_k\mid\cdot)$ is Gaussian with moments given by the predict--update recursion in Sec.~\ref{sec:structured_filtering}.
Specifically, we maximize
\begin{multline}
\label{eq:elbo_supp}
\cL_{\mathrm{ELBO}}
=
\E_{q(\sigma_{\mathrm{obs}},\sigma_{\mathrm{proc}},\bz_{1:T})}
\Bigg[
\sum_{k=1}^{T} \log p\!\left(\by_k \mid \bz_k,\sigma_{\mathrm{obs}}\right)
+ \log p(\bz_1)
+ \sum_{k=2}^{T} \log p\!\left(\bz_k \mid \bz_{k-1}, \sigma_{\mathrm{proc}}\right)
- \log q(\bz_{1:T} \mid \sigma_{\mathrm{obs}},\sigma_{\mathrm{proc}})
\Bigg] \\
-
\mathrm{KL}\!\left(q(\sigma_{\mathrm{obs}})\,\|\,p(\sigma_{\mathrm{obs}})\right)
-
\mathrm{KL}\!\left(q(\sigma_{\mathrm{proc}})\,\|\,p(\sigma_{\mathrm{proc}})\right),
\end{multline}
where $p(\sigma_{\mathrm{obs}})$ and $p(\sigma_{\mathrm{proc}})$ are specified by the priors in Level~3, $p(\bz_1)=\cN(\mathbf{0},\sigma_0^2 I_K)$, and
$p(\bz_k\mid\bz_{k-1},\sigma_{\mathrm{proc}})$ is the Gaussian transition approximation in Eq.\eqref{eq:em_transition}.
The likelihood term is evaluated on non-missing entries under the Gaussian observation model Eq.\eqref{eq:meas_coeff}.

\paragraph{Concrete forms of each term.}
We now expand the individual terms in Eq.\eqref{eq:elbo_supp}. All Gaussian log-densities below use the convention $\log\cN(\bx;\bmu,\bSigma) = -\tfrac{d}{2}\log(2\pi) - \tfrac{1}{2}\log|\bSigma| - \tfrac{1}{2}(\bx-\bmu)^\top\bSigma^{-1}(\bx-\bmu)$.

\textit{Observation log-likelihood.}
At time $t_k$, let $\cS_k$ denote the set of $n_k$ observed locations. Under Eq.\eqref{eq:meas_coeff}:
\begin{equation}
\label{eq:obs_ll}
\log p(\by_k \mid \bz_k, \sigma_{\mathrm{obs}})
=
-\frac{n_k}{2}\log(2\pi\sigma_{\mathrm{obs}}^2)
- \frac{1}{2\sigma_{\mathrm{obs}}^2}
\|\by_k - \bPhi_k \bz_k\|^2,
\end{equation}
where $\bPhi_k = \bPhi(\cS_k) \in \R^{n_k \times K}$ contains rows corresponding to observed locations only.

\textit{Transition log-density.}
Under the discretized transition Eq.\eqref{eq:em_transition}, for $k \geq 2$:
\begin{equation}
\label{eq:trans_ll}
\log p(\bz_k \mid \bz_{k-1}, \sigma_{\mathrm{proc}})
=
-\frac{K}{2}\log(2\pi\sigma_{\mathrm{proc}}^2 \Delta t_k)
- \frac{\|\bz_k - \hat{\bz}_k\|^2}{2\sigma_{\mathrm{proc}}^2 \Delta t_k},
\end{equation}
where $\hat{\bz}_k = \mathrm{ODESolve}(\bz \mapsto A_{\mathrm{NL}}\bz + g_\theta(\bz,t),\; \bz_{k-1},\; t_{k-1},\; t_k)$.

\textit{Prior log-density.}
\begin{equation}
\label{eq:prior_ll}
\log p(\bz_1)
=
-\frac{K}{2}\log(2\pi\sigma_0^2)
- \frac{\|\bz_1\|^2}{2\sigma_0^2}.
\end{equation}

\textit{Variational entropy.}
Since $q(\bz_{1:T}\mid\sigma_{\mathrm{obs}},\sigma_{\mathrm{proc}}) = \prod_{k=1}^{T} q(\bz_k\mid \cD_{1:k},\sigma_{\mathrm{obs}},\sigma_{\mathrm{proc}})$ with $q(\bz_k\mid\cdot) = \cN(\bm_k^+,\bP_k^+)$, where $(\bm_k^+,\bP_k^+)$ depend on $(\sigma_{\mathrm{obs}},\sigma_{\mathrm{proc}})$ through the filtering recursion, the conditional entropy decomposes as:
\begin{equation}
\label{eq:var_entropy}
-\E_q\!\left[\log q(\bz_{1:T}\mid\sigma_{\mathrm{obs}},\sigma_{\mathrm{proc}})\right]
=
\sum_{k=1}^{T} \frac{1}{2}\log|2\pi e\, \bP_k^+|
=
\sum_{k=1}^{T} \left(\frac{K}{2}\log(2\pi e) + \frac{1}{2}\log|\bP_k^+|\right).
\end{equation}

\section{Additional Experimental Details}
\label{app:experiments}

\subsection{Dataset Details}
\label{app:dataset_details}

\paragraph{D1: Advection--diffusion.}
We use constant velocity $\mathbf{v}=(0.5,\,0.3)$, diffusion coefficient $\kappa=0.1$, and forcing $F(x,t)$ given by two spatially localized Gaussian sources with oscillating amplitudes.
We discretize on a $64{\times}64$ grid with $\Delta t_{\mathrm{sim}}=0.01$ and save snapshots every $\Delta t_{\mathrm{save}}=1.0$, yielding $T=151$ frames over $[0,150]$.
Sensor observations are obtained by evaluating $u$ at $S=50$
irregularly placed coordinates and adding Gaussian noise with $\sigma_{\mathrm{noise}}=0.1$. The first $75\%$ of frames are used for training and the remainder for evaluation, with rolling-origin windows of context length $L{=}50$,
forecast horizon $H{=}10$, and stride $r{=}5$.

\paragraph{D2: Nonlocal IDE.}
We simulate on $D=[0,1]^2$ with periodic boundary conditions on a $64{\times}64$ grid, $T_{\mathrm{final}}=20$, and step size $\Delta t=0.1$, integrated with a fourth-order Runge--Kutta (RK4) method. We observe $S=36$ sensors on a regular spatial grid and add Gaussian noise with $\sigma_{\mathrm{noise}}=0.05$. The local diffusion coefficient is $\kappa=0.01$, and $f(x,t)$ consists of two spatially localized Gaussian sources with oscillating amplitudes.

The ground-truth kernel $G_\star$ has rank $R_\star=4$ with eigenvalues $\lambda_p = 1/(p+1)$.
The four modes $\psi_p$ are defined on normalized coordinates $(\tilde{x},\tilde{y}) \in [0,1]^2$:
\begin{alignat*}{2}
\psi_1 &= \sin(2\pi \tilde{y}) &\qquad& \text{(north--south coupling)},\\
\psi_2 &= \sin(2\pi \tilde{x}) && \text{(east--west coupling)},\\
\psi_3 &= \sin(2\pi \tilde{x})\sin(2\pi \tilde{y}) && \text{(four-corner interaction)},\\
\psi_4 &= \sin(4\pi \tilde{x})\cos(4\pi \tilde{y}) && \text{(multi-scale coupling)},
\end{alignat*}
each $L^2$-normalized over $D$. These modes create nonlocal interactions that couple spatially distant regions, mimicking teleconnection-like patterns where the evolution at one location depends on remote areas rather than only the immediate neighborhood. The first $75\%$ of time steps are used for training and the remainder for evaluation, with context length $L{=}50$, forecast horizon $H{=}10$, and stride $r{=}5$.

\paragraph{D3: EPA PM\textsubscript{2.5}.}
Daily mean PM\textsubscript{2.5} concentrations are obtained from the U.S.\ EPA Air Quality System (AQS) for the year 2024.
We select stations in the northeastern United States (latitude $35^\circ$- $45^\circ$N, longitude $70^\circ$--$85^\circ$W) with at least $80\%$ temporal coverage, yielding $S{=}50$ stations over $T{=}366$ days.
Values are capped at $150\,\mu\text{g}/\text{m}^3$ prior to the $\log(1{+}y)$ transform.
The natural missing rate is approximately $1$--$2\%$; we additionally introduce artificial missing-at-random during training. We use $L{=}5$ days and $H{=}1$ day with stride $r{=}1$.

\subsection{Metrics}
\label{app:metrics}

\paragraph{RMSE.}
Root mean squared error between the predictive mean and ground truth, averaged over all forecast windows, horizons, and sensors:
\[
\mathrm{RMSE}
= \sqrt{\frac{1}{|\mathcal{W}|}\sum_{w \in \mathcal{W}}
  \frac{1}{H \cdot S_w}\sum_{h=1}^{H}\sum_{s=1}^{S_w}
  \bigl(\hat{y}_{s,h}^{(w)} - y_{s,h}^{(w)}\bigr)^2},
\]
where $\mathcal{W}$ indexes rolling windows, $H$ is the forecast horizon, $S_w$ is the number of evaluated sensors in window $w$, and $\hat{y}$ denotes the predictive mean.

\paragraph{CRPS.}
The continuous ranked probability score (CRPS) measures the compatibility of a predictive CDF $F$ with the observed value $y$:
\[
\mathrm{CRPS}(F, y)
= \int_{-\infty}^{\infty}
  \bigl(F(z) - \mathbf{1}[z \ge y]\bigr)^2 dz.
\]
For \textsc{NLBST}, the predictive distribution is obtained by propagating $N{=}100$ posterior samples through the nonlinear ODE dynamics, yielding a distribution-free empirical CRPS that captures non-Gaussian effects from hyperparameter uncertainty and nonlinear propagation.
For baselines whose uncertainty estimates reduce to a predictive mean and variance (MC Dropout, VAE sampling), we compute CRPS under a Gaussian approximation in closed form via $\mathrm{CRPS}(\mathcal{N}(\mu,\sigma^2), y) = \sigma\bigl[\bar{z}\,(2\Phi(\bar{z})-1) + 2\varphi(\bar{z}) - \pi^{-1/2}\bigr]$, where $\bar{z} = (y-\mu)/\sigma$.

\subsection{Implementation Details}
\label{sec:impl_details}

\paragraph{Model configuration.}
We use $K{=}16$ RBF basis functions for \textbf{D1}, and $K{=}24$ Fourier basis functions for \textbf{D2} and \textbf{D3}. 
For \textbf{D2} and \textbf{D3}, we use a truncated Fourier basis on $[0,1]^d$ which is $L^2$-orthonormal. % and hence simplifies the Galerkin projection Theorem~\ref{thm:galerkin_main}. 
For the RBF basis (\textbf{D1}), basis centers are initialized from sensor coordinates, and the lengthscale is set based on inter-sensor distances; basis parameters are fixed during training for stability.
The Neural ODE residual $g_\theta$ is a 2-layer multi-layer perceptron (MLP) with hidden dimension 64 and tanh activation. ODE integration uses a fourth-order Runge--Kutta (RK4) method. All spatial coordinates are normalized to $[0,1]^d$ prior to basis construction and GP kriging. Min-max normalization is applied per sensor for all methods including \textsc{NLBST}.

\paragraph{Training and prediction.}
We optimize the ELBO via stochastic variational inference (SVI) with Adam (learning rate $10^{-3}$) and gradient clipping (max norm 1). \textsc{NLBST} is trained for 200 epochs with early stopping of patience 10. The variational guide performs full-covariance Kalman filtering with inline predict-update steps (Sec.~\ref{sec:inference}). For probabilistic forecasts, we draw $N{=}100$ Monte Carlo samples from the variational posterior and decode via the spatial basis. All results are averaged over 10 runs with different random seeds.

\subsection{Baseline Implementation Details}
\label{app:baselines}
All baselines are trained on the same data splits and evaluated under the same protocol as \textsc{NLBST}. Since most baselines were not originally designed for simultaneous irregular time series forecasting, spatial generalization, and uncertainty quantification, we adapt each method to provide predictions at unmeasured locations and calibrated predictive intervals.
Table~\ref{tab:baseline_capabilities} summarizes the key modeling capabilities of each method.

\begin{table}[h]
\caption{Modeling properties of baselines.
$\checkmark$: native support;
$\triangle$: partial or approximate support;
$\times$: not supported (adaptation required).}
\label{tab:baseline_capabilities}
\centering\small
\begin{tabular}{lcccc}
\toprule
Model & Spatially inductive & Irregular-time & Principled UQ & Nonlocal \\
\midrule
\textsc{NLBST} (Ours) & $\checkmark$ & $\checkmark$ & $\checkmark$ & $\checkmark$ \\
Linear-DSTM   & $\checkmark$ & $\times$ & $\checkmark$ & $\times$ \\
GRU-D         & $\times$     & $\checkmark$ & $\times$ & $\times$ \\
Latent ODE    & $\times$     & $\checkmark$ & $\triangle$ & $\times$ \\
FNO           & $\times$     & $\times$ & $\times$ & $\triangle$ \\
GraFITi       & $\times$     & $\checkmark$ & $\times$ & $\triangle$ \\
APN           & $\times$     & $\checkmark$ & $\times$ & $\times$ \\
\bottomrule
\end{tabular}
\end{table}

\paragraph{Linear-DSTM \citep{wikle2010general}.}
A discrete-time linear dynamic spatio-temporal model
$\bz_t = A\bz_{t-1} + \beps_t$,
$\by_t = \bPhi\bz_t + \mathbf{\nu}_t$,
using the same spatial basis as \textsc{NLBST}.
Inference includes the standard Kalman filter. This baseline isolates the contribution of nonlinear dynamics and continuous-time propagation.

\paragraph{GRU-D \citep{che2018recurrent}.}
GRU with trainable exponential decay on both input and hidden states to handle irregular time gaps. Observations are imputed via a learned convex combination
of the last observed value and the empirical mean, weighted by a time-decay gate.
Since GRU-D operates on a fixed sensor set, spatial prediction at unmeasured locations uses GP kriging with an RBF kernel. We used MC dropout for uncertainty estimation.

\paragraph{Latent ODE \citep{rubanova2019latent}.}
An ODE-RNN encoder processes observations in forward time to produce a variational posterior $q(\bz_0)$; latent states are then decoded via a Neural ODE.
Uncertainty is obtained by sampling from the VAE posterior. Spatial prediction at unmeasured locations uses GP kriging of the decoded measured-location outputs.

\paragraph{FNO \citep{li2020fourier}.}
A Fourier Neural Operator with 1-D spectral convolution layers operating along the temporal axis over all measured sensors jointly. At each layer, the input passes through a spectral convolution (FFT $\to$ linear mixing in frequency $\to$ iFFT) and a pointwise linear transform, combined via a residual connection.
Forecasting is autoregressive: at each horizon step, the model appends its own prediction and re-applies the Fourier blocks. FNO achieves global mixing via spectral convolution but does not learn an explicit nonlocal kernel. Uncertainty is provided via MC dropout, and spatial prediction uses GP kriging.

\paragraph{GraFITi \citep{yalavarthi2024grafiti}.}
A bipartite graph attention model that represents observations as edges connecting time nodes and channel (sensor) nodes. Missing observations correspond to absent edges, so irregular and incomplete data are handled natively.
Message passing alternates between channel$\to$time and time$\to$channel attention, with edge embeddings updated at each layer. We adapt GraFITi for forecasting by marking future time--channel positions as target edges with unknown values; the model predicts these entries in a single forward pass. GraFITi achieves implicit cross-channel coupling via bipartite attention but lacks an explicit spatial nonlocal mechanism. Uncertainty is obtained via MC dropout. Spatial prediction at unmeasured locations uses GP kriging with an RBF kernel whose length scale is set to 0.3.

\paragraph{APN \citep{liu2026apn}.}
A patch-based architecture using Time-Aware Patch Aggregation (TAPA): per-variable soft patches with learnable sigmoid boundaries aggregate context observations, followed by cross-patch attention and a time-query decoder.
The decoder conditions on learned time encodings to produce predictions at arbitrary future times. Uncertainty is obtained via MC dropout. Spatial prediction at unmeasured locations uses GP kriging with the same RBF kernel configuration as GraFITi.

\paragraph{Spatial interpolation for transductive baselines.}
All baselines except Linear-DSTM are spatially transductive: they produce outputs only at the training sensor locations. For these methods, predictions at unmeasured locations are obtained via GP kriging weights $\bW = K(\cS_u, \cS_m)\,K(\cS_m, \cS_m)^{-1}$ with an RBF kernel (length scale 0.3) and a nugget of $10^{-2}$ for numerical stability. In contrast, \textsc{NLBST} and Linear-DSTM predict at
unmeasured locations by evaluating the spatial basis at new coordinates, requiring no post-hoc interpolation.

\paragraph{Uncertainty quantification.}
Linear-DSTM and \textsc{NLBST} provide principled Bayesian uncertainty via the Kalman filtering posterior. Latent ODE obtains uncertainty through VAE posterior sampling. GRU-D, FNO, GraFITi, and APN use MC dropout (50 forward passes) to approximate predictive distributions.

\paragraph{Training details.}
All baselines are trained with Adam (lr $10^{-3}$) and gradient clipping (max norm 1.0).
For fair comparison, all methods use the same context length $L$ and forecast horizon $H$ as \textsc{NLBST} within each experimental setting. All models are trained for 200 epochs with early stopping of patience 10, and use the same preprocessing as \textsc{NLBST}. All results are averaged over 10 runs with different random seeds.

\section{Additional Experimental Results}
\label{app:additional_results}

\subsection{PM\textsubscript{2.5} Supplementary Results}

Table~\ref{tab:pm25_missing_supp} reports the robustness of \textsc{NLBST} to increasing missing rates on \textbf{D3} with $40\%$ spatial holdout. RMSE remains stable across $\rho_m \in \{0, 10, 20, 30\}\%$ for both measured and unmeasured locations, confirming that the \textsc{NLBST} gracefully handles observation sparsity by widening the prior uncertainty when measurements are absent.

Table~\ref{tab:pm25_rq2_supp} examines spatial generalization on \textbf{D3} as the fraction of held-out stations increases. Performance at unmeasured locations remains comparable to measured locations even at $40\%$ holdout, demonstrating that the spatial basis provides effective inductive prediction to unseen coordinates on real-world data.

\begin{table}[h]
\centering
\caption{(\textbf{RQ1}) Robustness to missing data on \textbf{D3} (PM\textsubscript{2.5}) (\textsc{NLBST}, $40\%$ spatial holdout).
RMSE and CRPS in $\log(1{+}y)$ space (mean $\pm$ std over 10 seeds).}
\label{tab:pm25_missing_supp}
\small
\begin{tabular}{c ccc}
\toprule
$\rho_m$ & RMSE (meas.) & RMSE (unmeas.) \\
\midrule
$0\%$  & $0.554 \pm 0.045$ & $0.558 \pm 0.053$  \\
$10\%$ & $0.554 \pm 0.047$ & $0.558 \pm 0.054$  \\
$20\%$ & $0.551 \pm 0.044$ & $0.555 \pm 0.052$  \\
$30\%$ & $0.546 \pm 0.045$ & $0.552 \pm 0.052$  \\
\bottomrule
\end{tabular}
\end{table}

\begin{table}[h]
\centering
\caption{(\textbf{RQ2}) Spatial generalization on \textbf{D3} (PM\textsubscript{2.5}) (\textsc{NLBST}).
RMSE in $\log(1{+}y)$ space (mean $\pm$ std over 10 seeds).}
\label{tab:pm25_rq2_supp}
\small
\begin{tabular}{c cc}
\toprule
Holdout \% & RMSE (measured) & RMSE (unmeasured) \\
\midrule
$0\%$  & $0.567 \pm 0.003$ & --- \\
$20\%$ & $0.542 \pm 0.026$ & $0.523 \pm 0.025$ \\
$40\%$ & $0.552 \pm 0.047$ & $0.557 \pm 0.053$ \\
\bottomrule
\end{tabular}
\end{table}

\subsection{Dynamics Ablation (RQ4)}
Table~\ref{tab:ablation_supp} isolates the contribution of each component in the \textsc{NLBST} dynamics on \textbf{D2} and \textbf{D3}. On \textbf{D2}, removing the nonlocal coupling (neural only) degrades CRPS most severely, while removing the neural residual (linear only) primarily affects RMSE, indicating that the two components play complementary roles in prediction accuracy and uncertainty calibration. The same pattern holds on the real-world \textbf{D3} dataset: both ablations increase measured-location RMSE relative to the full model ($0.480 \to 0.495/0.496$), confirming that both $A_{\mathrm{NL}}$ and $g_\theta$ provide complementary performance gains rather than introducing redundant complexity.

\begin{table}[h]
\caption{(\textbf{RQ3}, \textbf{RQ4}) Dynamics ablation on \textbf{D2} (Nonlocal IDE) and \textbf{D3} (PM\textsubscript{2.5}). Mean (std) over 10 seeds; \textbf{D3} reports RMSE at measured locations. Best per column in \textbf{bold}.}
\label{tab:ablation_supp}
\centering\small
\setlength{\tabcolsep}{5pt}
\begin{tabular}{lccc}
\toprule
& \multicolumn{2}{c}{\textbf{D2}} & \textbf{D3} \\
\cmidrule(lr){2-3}\cmidrule(lr){4-4}
Dynamics & RMSE$\downarrow$ & CRPS$\downarrow$ & RMSE$\downarrow$ \\
\midrule
$A_{\mathrm{NL}}\bz + g_\theta$ (full)
  & \textbf{9.54\,(3.36)} & \textbf{6.84\,(2.26)} & \textbf{0.480\,(0.010)} \\
$A_{\mathrm{NL}}\bz$ (linear only)
  & 10.85\,(2.70) & 7.20\,(1.40) & 0.495\,(0.006) \\
$g_\theta(\bz,t)$ (neural only)
  & 10.84\,(2.73) & 8.03\,(1.42) & 0.496\,(0.006) \\
\bottomrule
\end{tabular}
\end{table}

\subsection{Long-Horizon Forecasting}

To assess robustness beyond the short-horizon setting in the main experiments (Table~\ref{tab:pm25_rq1}, $H{=}1$), we evaluate long-horizon forecasting on \textbf{D3} with $H \in \{10,30,50\}$. The setup matches the main D3 experiments, except that we use $T_{\mathrm{train}}{=}200$ and $T_{\mathrm{test}}{=}166$ to support evaluation up to $H{=}50$, with the context window scaling with $H$ ($14, 30, 50$). Means over 3 seeds are reported in Table~\ref{tab:long_horizon}. \textsc{NLBST} consistently achieves the lowest RMSE and CRPS across all horizons. Notably, Linear-DSTM---which shares the same spatial basis and Kalman updates but uses discrete-time linear dynamics---degrades catastrophically as $H$ grows (RMSE $1.43 \to 19.17 \to 63.23$), whereas \textsc{NLBST} remains stable (RMSE $\le 0.59$ throughout). This highlights that continuous-time nonlinear dynamics are essential for stable extrapolation over long horizons.

\begin{table}[h]
\caption{(\textbf{RQ1}, \textbf{RQ3}) Long-horizon forecasting on \textbf{D3}
(PM\textsubscript{2.5}) with $H \in \{10,30,50\}$. Mean over 3 seeds; RMSE and CRPS in $\log(1{+}y)$ space. Best per column in \textbf{bold}.}
\label{tab:long_horizon}
\centering\small
\setlength{\tabcolsep}{5pt}
\begin{tabular}{l cc cc cc}
\toprule
& \multicolumn{2}{c}{$H{=}10$} & \multicolumn{2}{c}{$H{=}30$} & \multicolumn{2}{c}{$H{=}50$} \\
\cmidrule(lr){2-3}\cmidrule(lr){4-5}\cmidrule(lr){6-7}
Model & RMSE & CRPS & RMSE & CRPS & RMSE & CRPS \\
\midrule
\textsc{NLBST} (Ours)
  & \textbf{0.552} & \textbf{0.322} & \textbf{0.590} & \textbf{0.347} & \textbf{0.490} & \textbf{0.286} \\
Linear-DSTM
  & 1.432 & 0.778 & 19.165 & 9.106 & 63.225 & 35.082 \\
GRU-D
  & 1.260 & 1.100 & 1.314 & 1.151 & 1.238 & 1.080 \\
Latent ODE
  & 1.287 & 0.906 & 1.232 & 0.848 & 1.304 & 0.925 \\
FNO
  & 1.264 & 0.748 & 1.288 & 0.781 & 1.407 & 0.830 \\
GraFITi
  & 0.977 & 0.779 & 0.863 & 0.684 & 0.888 & 0.708 \\
APN
  & 0.781 & 0.579 & 0.953 & 0.765 & 0.787 & 0.595 \\
\bottomrule
\end{tabular}
\end{table}
\end{document}